\documentclass{article}
\usepackage{ijcai15}

\usepackage{times}

\usepackage{latexsym}
\usepackage{hyperref}
\usepackage{tikz}
\usepackage{url}
\usepackage{float}
\usepackage{amssymb}

\usepackage{xcolor}

\newcommand{\alias}[2]{\newcommand{#1}[0]{#2}}
\newcommand{\ralias}[2]{\renewcommand{#1}[0]{#2}}
\newcommand{\malias}[2]{\alias{#1}{\ensuremath{#2}}}
\newcommand{\mralias}[2]{\ralias{#1}{\ensuremath{#2}}}

\newtheorem{lemma}{Lemma}
\newtheorem{proposition}{Proposition}
\newtheorem{corollary}{Corollary}
\newtheorem{theorem}{Theorem}
\newtheorem{example}{Example}
\newtheorem{claim}{Claim}

\newenvironment{proof} {\noindent\emph{Proof:}}{$\left.\right.$\hfill$\Box$}

\malias{\UP}{\mathbb P}
\malias{\UQ}{\mathbb Q}
\alias{\patquery}{union of patterned conjunctive queries}
\alias{\Patquery}{Union of patterned conjunctive queries}
\alias{\conjpatquery}{patterned conjunctive query}
\alias{\Conjpatquery}{Patterned conjunctive query}

\malias{\x}{\vec x}
\mralias{\t}{\vec t}
\mralias{\v}{\vec v}
\malias{\iv}{\vec i}
\mralias{\j}{\vec j}
\mralias{\k}{\vec k}
\mralias{\l}{\vec l}

\malias{\R}{\mathcal R}
\malias{\F}{\mathcal F}
\malias{\Q}{\mathcal Q}

\newcommand{\fun}[1]{\ensuremath{\mbox{\sl #1}}} 

\def\ltrans{linear+trans}

\def\PSpace{\textsc{PSpace}}

\def\NL{\textsc{NL}}

\title{Combining Existential Rules and Transitivity: Next Steps}

\author{Jean-Fran\c cois Baget\\INRIA \& \\ Univ. Montpellier II \\ Montpellier, France
\And
Meghyn Bienvenu\\ CNRS \& \\ Univ. Paris-Sud\\Orsay, France
\And
Marie-Laure Mugnier \\INRIA \& \\Univ. Montpellier II \\ Montpellier, France
\And
Swan Rocher \\INRIA \& \\ Univ. Montpellier II \\ Montpellier, France
}

\begin{document}

\maketitle

\begin{abstract}
We consider existential rules (aka Datalog$\pm$) as a formalism for specifying ontologies.
In recent years, 
many classes of existential rules have been exhibited for which conjunctive query (CQ) entailment is decidable.
However, most of these classes cannot express transitivity of binary relations, a 
 frequently used
 modelling construct.
In this paper\footnote{This document is an extended and revised version of the IJCAI'16 paper with the same title. It contains an appendix with proofs omitted from the conference version. The most recent revision (December 2016) adds a new result about \emph{msa} (Proposition \ref{msa}) and corrects a mistake in the complexity bounds in Theorem \ref{thm:combined}.}, we address 
 the issue of whether transitivity can be safely combined with decidable classes of existential rules.
First, we prove that transitivity is incompatible with one of the simplest decidable
classes, namely aGRD (acyclic graph of rule dependencies), 
which clarifies the landscape of `finite expansion sets' of rules.
Second,
we show that transitivity can be safely added to linear rules (a subclass of guarded rules, which generalizes the description logic DL-Lite$_R$)
in the case of atomic CQs, and also for general CQs if we place a
minor syntactic restriction on the rule set.
This is shown by means of a novel query rewriting algorithm that is
specially tailored to handle transitivity rules.
 Third, 
 for the identified decidable cases,
 we analyze 
 the combined and data complexities of query entailment. 
\end{abstract}

\section{Introduction}

Ontology-based data access (OBDA)
 is a new paradigm in data management, which exploits the semantic information provided by ontologies when querying data.
 Briefly, the notion of a database is replaced by that of a knowledge base (KB), composed of a dataset and an ontology.
 \emph{Existential rules}, aka Datalog$\pm$, have been proposed to represent ontological knowledge in this context \cite{cgl09,blms09,bmrt11,kr11}.
 These rules are an extension of function-free first-order
 Horn rules (aka Datalog), that allows for existentially
 quantified variables in rule heads.
 The addition of existential quantification allows one 
 to assert
 the existence of yet unknown entities and to reason about
 them, an essential feature of ontological languages, which
 is also at the core of description logics (DLs).
 Existential rules generalize the DLs most often considered
in the OBDA setting,
 like the DL-Lite and $\mathcal {EL}$ families
 \cite{dl-lite07,baader03,ltw09} and 
 Horn DLs \cite{krh07}.

The fundamental decision problem related to OBDA is the following:
is a Boolean
conjunctive query (CQ) entailed from a KB?
This problem has long been known to be undecidable for general existential rules
(this follows e.g., from \cite{beeri-vardi81}). Consequently, 
a significant amount of research has been devoted to the issue of 
finding decidable subclasses with a good expressivity / tractability
tradeoff.
It has been observed that most exhibited decidable classes  fulfill one of the
three following properties \cite{blms11}:
finiteness of a forward chaining mechanism known
as the chase, which allows inferences  to be materialized 
 in the data
(we call such rule sets \emph{finite expansion sets}, \emph{fes});
finiteness of  query rewriting into a union of CQs, which
allows to the rules to be  compiled into the query (\emph{finite unification sets, fus});
tree-like shape of the possibly infinite chase,
 which allows one to finitely encode the result
(\emph{bounded-treewidth sets, bts)}. The class of \emph{guarded } rules \cite{cgk08}
is a well-known class satisfying the latter property.

Known decidable classes are able to express many useful properties of binary relations
(e.g., 
inverses / symmetry) 
but most of them lack the ability to define a frequently required property,
namely \emph{transitivity}. This limits their applicability in key application
areas like biology and medicine, for which transitivity of binary relations (especially the ubiquitous `part of' relation)
is an essential modelling construct.
The importance of transitivity has long been acknowledged in the DL community \cite{DBLP:journals/logcom/HorrocksS99,DBLP:conf/ecai/Sattler00},
and many DLs support transitive binary relations. 
While adding transitivity to a DL often does not increase the complexity of CQ entailment (see \cite{DBLP:conf/ijcai/EiterLOS09} for some exceptions),
it is known to complicate the design of query answering procedures \cite{DBLP:journals/jair/GlimmLHS08,DBLP:conf/aaai/EiterOSTX12}, 
due to the fact that it destroys
the tree structure of the chase upon which DL reasoning algorithms typically rely.
In contrast to the extensive literature on transitivity in DLs, 
rather  little is known about the compatibility of transitivity with decidable classes
of existential rules.\footnote{Since the conference version of this paper, the compatibility of 
transitivity with frontier-one rules (a \emph{bts} class that has close connections to Horn DLs) has been shown \cite{amarilli16}.}
A notable exception is the result  of \cite{gpt13} on
the incompatibility of transitivity with guarded rules, which holds
even under strong syntactic restrictions (see Section 3).

In this paper, we 
investigate the issue of whether transitivity can be safely added to some
well-known rule classes and provide three main contributions.
First, we show that
adding transitivity 
to one of the simplest \emph{fes} and \emph{fus} classes (namely \emph{aGRD}) makes atomic CQ entailment undecidable (Theorem \ref{th-agrd}).
We also provide (un)decidability results for the classes \emph{swa} and \emph{msa} extended with transitivity, which yields a complete picture of the impact of transitivity on known \emph{fes} classes.
%
Second, we investigate the impact of adding transitivity to \emph{linear} rules,
a natural subclass of guarded rules which generalizes the well-known description logic DL-Lite$_R$.
We introduce a query rewriting procedure that is sound and complete for all
rule sets consisting of linear and transitivity rules (Theorem  2),
and which is guaranteed to terminate 
for atomic CQs, and for arbitrary CQs if the rule set contains only unary and binary predicates or satisfies a certain safety condition, yielding
decidability for these cases (Theorem 3).
Third, based on a careful analysis
of our algorithm, 
we establish upper and lower bounds on the combined and data
complexities of query entailment for the identified decidable cases (Theorems \ref{thm:data} and \ref{thm:combined}). 
While the addition of transitivity leads to an increase in the combined complexity of atomic CQ entailment (which rises from \PSpace-complete to ExpTime-complete),
the obtained data complexity is the lowest that could be expected,
namely, \NL-complete. 

\section{Preliminaries}

\label{sec:prelims}

A \emph{term} is a variable or a constant. An \emph{atom} is of the form $p(t_1, \ldots, t_k)$ where $p$ is a predicate of arity
$k$, and the $t_i$ are terms. We consider \emph{(unions of) Boolean conjunctive queries ((U)CQs)},
which are (disjunctions of) existentially closed conjunctions of atoms.
Note however that all results can be extended to non-Boolean queries.
A CQ is often viewed as the \emph{set} of atoms.
An \emph{atomic CQ} is a CQ consisting of a single atom.
A \emph{fact} is an atom without variables. A \emph{fact base} is a finite set of facts.

An \emph{existential rule} (hereafter abbreviated to  \emph{rule})  $R$
 is a formula
$ \forall \vec x \forall \vec y(B[ \vec x,\vec y] \rightarrow \exists \vec z ~H[\vec x,\vec z])$
where $B$  and $H$ are conjunctions of atoms, resp. called the \emph{body}
and the \emph{head} of $R$. The variables  $\vec z$ (resp. $\vec x$),
which occur only in $H$ (resp. in $B$ and in $H$) are called  \emph{existential} variables (resp. \emph{frontier} variables).
Hereafter, we  omit quantifiers in rules and simply denote a rule by  $B \rightarrow H$.
For example, $p(x,y) \rightarrow p(x,z)$
stands for $\forall x \forall y (p(x,y) \rightarrow \exists z (p(x,z)))$.
A \emph{knowledge base} (KB) $\mathcal K = (\F, \mathcal R)$ consists of 
a fact base $\F$ and a finite set of 
rules $\mathcal R$.
The \emph{(atomic) CQ entailment} problem consists in deciding 
 whether $\mathcal K \models Q$, where  
$\mathcal K$ is a KB viewed as a first-order theory, $Q$ is {an (atomic) CQ, and $\models$ denotes standard
 logical entailment.

Query rewriting relies on a unification operation between the query and a rule head.
Care must be taken when handling existential variables:
when a term $t$ of the query is unified with an existential variable in a rule head, all atoms in which $t$ occurs must also be part of the unification, otherwise the result is unsound.  Thus, 
instead of unifying one query atom 
 at a time, 
we have to unify subsets (``pieces'') of the query, hence the notion of a piece-unifier defined next.
 A partition $P$ of a set of terms is said to be
 \emph{admissible} if no class of $P$ contains two constants; 
  a substitution $\sigma$ can be obtained from $P$ by selecting an element $e_i$ in each class $C_i$ of  $P$, with priority given to constants, and setting $\sigma(t) = e_i$ for all $t \in C_i$.
A \emph{piece-unifier} of a CQ $Q$ with a rule $R = B \rightarrow H$ is a triple $\mu = (Q',H', P_{\mu})$, where $Q' \subseteq Q$, $H' \subseteq H$  and $P_{\mu}$ is an admissible partition on the terms of $Q' \cup H'$
such that:
\begin{enumerate}
\item $\sigma(H') = \sigma(Q')$, where $\sigma$ is any substitution obtained from $P_{\mu}$;
\item   if a class $C_i$ in $P_{\mu}$ contains an existential variable,
	then the other terms in $C_i$ are variables from $Q'$ that do not occur in $(Q \setminus Q')$.
\end{enumerate}	
We say that $Q'$ is a \emph{piece} (and $\mu$ is a single-piece unifier) if there is no non-empty subset $Q''$ of $Q'$ such that $P_{\mu}$ restricted to $Q''$ satisfies
Condition~2.
From now on, we consider only single-piece unifiers, which we simply call \emph{unifiers}.
The \emph{(direct) rewriting} of $Q$ with $R$ w.r.t.\ $\mu$ is $\sigma(Q \setminus Q') \cup \sigma(B)$ where
$\sigma$ is a substitution obtained from $P_{\mu}$. A \emph{rewriting }of $Q$ w.r.t.\ a set of rules
$\mathcal R$ is a CQ obtained by a sequence $Q = Q_0, \ldots, Q_n$ ($n \geq 0$) where for all $i > 0$,
$Q_i$ is a direct rewriting of $Q_{i-1}$ with a rule from $\mathcal R$. For any fact base $\F$, we have
that $\F, \mathcal R \models Q$ iff there is a rewriting $Q_n$ of $Q$ w.r.t.\ $\mathcal R$ such
that $\F \models Q_n$ \cite{klmt13}.

\begin{example}
Consider the rule $R= h(x) \rightarrow p(x,y)$
and CQ $Q = q(u) \wedge p(u,v) \wedge p(w,v) \wedge r(w)$.
If $p(u,v)$ is unified with $p(x,y)$, then $v$ is unified with the existential variable $y$, hence $p(w,v)$ has to be part of the unifier.
The triple $\mu = (\{p(u,v), p(w,v)\}, \{p(x,y)\}, \{\{x,u,w\}\{v,y\}\}$ is a unifier.
The direct rewriting of $Q$ associated with the substitution
$\sigma = \{x \mapsto u, w \mapsto u, y \mapsto v\}$ is $h(u) \wedge q(u) \wedge r(u)$.
\end{example}

We now define some important kinds of rule sets (see e.g., \cite{mugnier11} for an overview). A model $M$ of a KB $\mathcal K$ is called
\emph{universal} if for any CQ $Q$, $M$ is a model of $Q$ iff $\mathcal K \models Q$.  A rule set $\mathcal R$
is a \emph{finite expansion set (fes)} if any KB $(\F,\mathcal R)$ has a finite universal model. It
is a \emph{bounded-treewidth set (bts)} if any KB $(\F,\mathcal R)$ has a (possibly infinite)
universal model of bounded treewidth. It is a \emph{finite unification set} \emph{(fus)} if, for any CQ $Q$, there is a finite set $S$
of rewritings of $Q$ w.r.t.\ $\mathcal R$ such that for any fact base $\F$, we have $\F, \mathcal R \models
Q$ iff there is $Q' \in S$ such that $\F \models Q'$.

 A \emph{Datalog} rule has no existential variables, hence Datalog rule sets are \emph{fes}.
 Other kinds of \emph{fes} rules are considered in the next section. A rule $B \rightarrow H$ is \emph{guarded} if there is an atom in $B$ that contains all the variables occurring in $B$. Guarded rules are \emph{bts}. A \emph{linear} rule has a body composed of a single atom and does not contain any constant. Linear rules are guarded, hence \emph{bts}, moreover they are \emph{fus}.

 As a special case of Datalog rules, we have \emph{transitivity rules}, of the form $p(x,y) \wedge p(y,z) \rightarrow p(x,z)$, which are not \emph{fus}. A predicate is called \emph{transitive} if it appears in a transitivity rule.
If $\mathcal{C}$ is a class of rule sets, $\mathcal C$+$\fun{trans}$ denotes the class obtained by adding transitivity rules to rule sets from $\mathcal C$.

\section{Combining \emph{fes / fus} and Transitivity}

A large hierarchy of \emph{fes} classes is known (see e.g., \cite{chkk13} for an overview). Beside Datalog, the simplest classes are \emph{weakly-acyclic (wa)} sets, which prevent cyclic propagation of existential variables along predicate positions, and \emph{aGRD} (acyclic Graph of Rule Dependencies) sets, which prevent cyclic dependencies between rules. Datalog is generalized by \emph{wa}, while \emph{wa} and \emph{aGRD} are incomparable. Some classes generalize \emph{wa} by a finer analysis of variable propagation (up to \emph{super-weakly acyclic (swa)} sets). Most other \emph{fes} classes generalize both \emph{wa} and \emph{aGRD}.

We show that \emph{aGRD}+\fun{trans} is undecidable even for atomic CQs. Since \emph{aGRD} is both \emph{fes} and \emph{fus}, this negative result also transfers to \emph{fes}+\fun{trans} and \emph{fus}+\fun{trans}.

\begin{theorem}\label{th-agrd}
Atomic CQ entailment over \emph{aGRD}+\fun{trans} KBs is undecidable, even with a single transitivity rule.
\end{theorem}
\begin{proof}
The proof is by reduction from atomic CQ entailment with general existential rules (which is known to be undecidable).
Let $\R$ be a set of rules. We first translate $\R$ into an  \emph{ aGRD} set of rules $\R^a$.
We consider the following new predicates: $p$ (which will be the transitive predicate) and, for each rule $R_i \in \R$, predicates $a_i$ and $b_i$.
Each rule $R_i = B_i \rightarrow H_i$ is translated into the two following rules:
\begin{itemize}
	\item $R^{1}_i = B_i \rightarrow a_i(\x,z_1) \wedge p(z_1,z_2) \wedge p(z_2,z_3) \wedge b_i(z_3)$
	\item $R^{2}_i = a_i(\x,z_1) \wedge p(z_1,z_2) \wedge b_i(z_2) \rightarrow H_i$
\end{itemize}
where $z_1$,$z_2$ and $z_3$ are existential variables and $\x$ are the variables in $B_i$.

Let $\R^a = \{R^1_i,R^2_i~|~R_i \in \R\}$, and let $GRD(\R^a)$ be the graph of rule dependencies of $\R^a$, defined as follows: the nodes of $GRD(\R^a)$ are in bijection with $\R^a$, and there is an edge from a node $R_1$ to a node $R_2$ if the rule $R_2$ depends on the rule $R_1$, i.e., if there is a piece-unifier of the body of $R_2$ (seen as a CQ) with the head of $R_1$.

We check that for any $R_i \in \R$,  $R^1_i$ has no outgoing edge and $R^2_i$ has no incoming edge (indeed the $z_j$ are existential variables).
Hence, in $GRD(\R^a)$ all (directed) paths are of length less or equal to one. It follows that
$GRD(\R^a)$ has no cycle, i.e., $\R^a$ is aGRD.

Let  $R^{t}$ be the rule stating that $p$ is transitive.
 Let  $\R' = \R^a \cup \{R^t\}$. The idea is that $R^t$ allows to ``connect" rules in $\R^a$ that correspond to the same rule in $\R$.
For any fact base $\mathcal F$ (on the original vocabulary), for any sequence of rule applications from $\mathcal F$ using rules in $\R$, one can build a sequence of rule applications from $\mathcal F$ using rules from $\R'$, and reciprocally, such that both sequences produce the same fact base (restricted to atoms on the original vocabulary).
Hence, for any $\mathcal F$ and $Q$ (on the original vocabulary),  we have that $\mathcal F, \R \models Q$ iff $\mathcal F, \R' \models Q$.
\end{proof}

\begin{corollary}
Atomic CQ entailment over \emph{fus}+\fun{trans} or \emph{fes}+\fun{trans} KBs is undecidable.
\end{corollary}

Most known \emph{fes} classes that do not generalize \emph{aGRD} range between Datalog and \emph{swa} (inclusive). It can be easily checked that any \emph{swa} set of rules remains \emph{swa} when transitivity rules are added (and this is actually true for all known classes between Datalog and \emph{swa}).

\begin{proposition}
The classes \emph{swa} and  \emph{swa+\fun{trans}} coincide. 
Hence, \emph{swa+\fun{trans}} is decidable.
\end{proposition}

\begin{proof}
It suffices to note that the addition of transitivity rules does not create new edges in the `SWA position graph'  from \cite{chkk13}.
\end{proof}
\smallskip

The only remaining \emph{fes} class that is not covered by the preceding results, namely Model
Summarizing Acyclicity (\emph{msa}) from \cite{chkk13},
can be shown to be
incompatible with transitivity rules:


\begin{proposition}\label{msa}
Atomic CQ entailment over \emph{msa}+\fun{trans} KBs is
	undecidable.
\end{proposition}

%
%
%
%
%
%
%

It follows that the effect of transitivity on the currently known \emph{fes} landscape is now quite clear, which is not the case for  \emph{fus} classes.
In the  following, we focus on a well-known \emph{fus} class, namely \emph{linear} rules.
We show by means of a query rewriting procedure that query entailment over \emph{linear}+\fun{trans} KBs is decidable 
in the case of atomic CQs, as well as for general CQs if we place a
minor safety condition on the rule set. 
Such an outcome was not obvious in the light of 
existing results.
Indeed,
atomic CQ entailment over \emph{guarded}+\fun{trans} rules was recently shown
undecidable, even when restricted to rule sets that belong to the two-variable fragment, use only unary and binary predicates, and contain only two transitive predicates
 \cite{gpt13}. Moreover, inclusion
dependencies (a subclass of linear rules) and functional
dependencies (a kind of rule known to destroy tree
structures, as do transitivity rules) are known to be
incompatible \cite{chandra-vardi85}.

\section{Linear Rules and Transitivity}

To obtain finite representations of sets of rewritings involving transitive predicates, we define a framework based on the notion of  \emph{pattern}.
\subsection{Framework} To each transitive predicate we assign a \emph{pattern name}.
Each pattern name has an associated \emph{pattern definition }
$P := a_1 | \dots | a_k$, where each $a_i$ is an atom that contains the special variables $\#1$ and $\#2$.
A \emph{pattern} is either a \emph{standard pattern} $P[t_1,t_2]$ or a \emph{repeatable pattern} $P^+[t_1,t_2]$,
where $P$ is a pattern name and $t_1$ and $t_2$ are terms.
A \emph{ \patquery{}} (UPCQ) is a pair $(\UQ, \UP)$, where
$\UQ$ is a disjunction of conjunctions of atoms and patterns,
and $\UP$ is a set of pattern definitions that gives
a unique definition to each pattern name occurring in $\UQ$. A
{\em \conjpatquery} (PCQ) $\Q$ is a UPCQ without disjunction.
For the sake of simplicity, we will often denote a (U)PCQ by its first
component $\UQ$, leaving the pattern definitions implicit.

An \emph{instantiation} $T$ of a UPCQ $(\UQ, \UP)$ is a node-labelled tree
that satisfies the following conditions:
\begin{itemize}
\item the root of $T$ is labelled by $\Q \in \UQ$; 
\item 
the children of the root are labelled by the patterns and atoms occurring in $\Q$; 
\item 
 each node that is labelled by a repeatable pattern $P^+[t_1,t_2]$ may be 
expanded into $k \geq 1$ children
 labelled  respectively by $P[t_1,x_1]$, $P[x_1,x_2], \dots$, $P[x_{k-1},t_2]$,
 where the $x_i$ are fresh variables;
\item 
each node labelled by a standard pattern $P[t_1,t_2]$ may be expanded into a single child whose label is
obtained from an atom $a$ in the
pattern definition of $P$ in $\mathbb{P}$ by substituting $\#1$ (resp. $\#2$) by $t_1$ (resp. $t_2$), and freshly renaming
the other variables.
\end{itemize}
For brevity, we will often refer to nodes in an instantiation using their labels.

The {\em instance} associated with an instantiation 
is the PCQ obtained
by taking the conjunction of the labels of its leaves.
An instance of a UPCQ is an instance associated with one of its instantiations.
An instance is called \emph{full} if it does not contain any pattern, and we denote by
$full(\UQ,\UP)$ the set of  full instances of $(\UQ,\UP)$.

\begin{figure}[t]

\begin{center}
\begin{tikzpicture}[node distance=1.41cm,every node/.style={transform shape},scale=0.75]
	\node (f0) at(-3.75,0.25) {};
	\node (f1) at(-3.75,-2.25) {};
	\node (f2) at(3.75,0.25) {};
	\node (f3) at(3.75,-2.25) {};

	\node (q)   at(0,0)       {$\Q$};
	\node (p1)  at(-3,-1)   {$P_1^+[a,z]$};
	\node (p2)  at(0,-1)      {$P_2^+[z,b]$};
	\node (s)   at(3,-1)    {$s_1(a,b)$};

	\node (p11) at(-3,-2)   {$P_1[a,z]$};

	\node (p21) at(-1,-2)  {$P_2[z,x_1]$};
	\node (p22) at(1,-2)   {$P_2[x_1,b]$};

	\node[node distance=1cm] (s1) [below of=p11] {$s_2(a,y_0,z)$};
	\node[node distance=1cm] (s2) [below of=p21] {$s_2(x_1,y_1,z)$};
	\node[node distance=1cm] (s3) [below of=p22] {$p_2(x_1,b)$};

	\path[draw=black] (q) -- (p1);
	\path[draw=black] (q) -- (p2);
	\path[draw=black] (q) -- (s);

	\path[draw=black] (p1) -- (p11);
	\path[draw=black] (p2) -- (p21);
	\path[draw=black] (p2) -- (p22);

	\path[draw=black] (p11) -- (s1);
	\path[draw=black] (p21) -- (s2);
	\path[draw=black] (p22) -- (s3);

	\path[densely dotted,draw=black] (f0.center) --
	(f1.center) -- (f3.center) -- (f2.center) -- (f0.center);
\end{tikzpicture}
\end{center}

	\caption{Instantiations of a PCQ}
	\label{fig:ex_instantiation}
\end{figure}
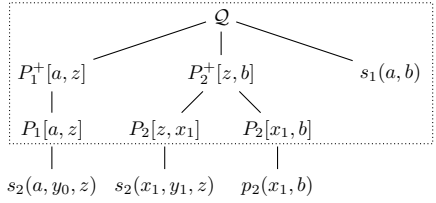

\begin{example}
\label{ex:instantiation}
	Let $(\Q, \UP)$ be a PCQ, where $\Q = P_1^+[a,z] \wedge P_2^+[z,b] \wedge s_1(a,b)$
	and $\UP$ contains the pattern definitions:
	$P_1 := p_1(\#1,\#2) | s_2(\#1,y,\#2)$ and $P_2 := p_2(\#1,\#2) |$ $ s_2(\#2,y,\#1)$.


	Two instantiations of $\Q$  are displayed in Figure \ref{fig:ex_instantiation}. The smaller instantiation (within the dotted lines)
	gives rise to the (non-full) instance $Q_1 = P_1[a,z] \wedge P_2[z,x_1] \wedge P_2[x_1,b] \wedge s_1(a,b)$. By expanding
	the three nodes labelled by patterns according to the definitions in $\UP$, we may obtain the larger instantation (occupying
	the entire figure), whose associated instance $Q_2 = s_2(a,y_0,z) \wedge s_2(x_1,y_1,z) \wedge
	p_2(x_1,b) \wedge s_1(a,b)$ is a full instance for $(\Q,\UP)$.

\end{example}

A UPCQ $(\UQ,\UP)$ can be translated into a set of Datalog rules $\Pi_\UP$ and a UCQ $Q_\UQ$ as follows.
For each 
definition $P := a_1(\t_1) | \dots | a_k(\t_k)$ in $\UP$,
we create the transitivity rule
$p^+(x,y) \wedge p^+(y,z) \rightarrow p^+(x,z)$
and the rules
$a_i(\t_i) \rightarrow p^+(\#1,\#2)$ ($1 \leq i \leq k$).
The UCQ $Q_\UQ$ is obtained from $\UQ$ by replacing each repeatable pattern $P^+[t_1,t_2]$
by the atom $p^+(t_1,t_2)$.
Observe that $\Pi_\UP$ is non-recursive except for the transitivity rules.
The next proposition states that $(\Pi_\UP,Q_\UQ)$ can be seen as a finite representation of the set of
full instances of
$(\UQ,\UP)$.

\begin{proposition}\label{datalog}
	Let $\F$ be a fact base and $(\UQ,\UP)$ be a UPCQ. 
	Then $\F,\Pi_\UP \models Q_\UQ$ iff $\F \models Q$ for some  $Q \in full(\UQ,\UP)$.
\end{proposition}

A \emph{unifier} $\mu = (Q',H,P_u)$ of a \emph{PCQ} 
is a unifier of one of its (possibly non-full) instances 
such that $Q'$ 
is a set of (usual) atoms. We distinguish two types of unifiers (internal and external),
defined next.

Let $T$ be an instantiation, $Q$ be its associated instance,
and $\mu = (Q',H,P_u)$ be a unifier of $Q$.
Assume $T$ contains a repeatable pattern 
$P^+[t_1,t_2]$ that is expanded into
$P[u_0,u_1],$ $ \ldots, P[u_k,u_{k+1}]$, where $u_0=t_1$ and $u_{k+1}=t_2$.
We call 
$P[u_i,u_{i+1}]$ \emph{relevant for $\mu$} if it is expanded
into an atom from $Q'$.
Because we consider only single-piece unifiers (cf.\ Sec.\ \ref{sec:prelims}), it follows that 
if such relevant patterns exist, 
they form a sequence $P[u_i,u_{i+1}]$,$P[u_{i+1},u_{i+2}]$,$\dots, P[u_{j-1},u_{j}]$.
Terms $u_i$ and $u_j$ are called {\em external} to $P^+[t_1,t_2]$ w.r.t.\ $\mu$;
the other terms occurring in the sequence are called {\em internal}.
The unifier $\mu$ is said to be {\em internal} if all atoms from $Q'$
are expanded from a single repeatable pattern, and no external
terms are unified together or with an existential variable;
otherwise $\mu$ is called \emph{external}.

\begin{example}
\label{ex:instantiation_unifiers}
	Consider 
	$Q_2$ from 
	Example \ref{ex:instantiation} and the rules $R_1 = s_1(x',y') \rightarrow p_2(x',y')$ and
	$R_2 = s_1(x',y') \rightarrow s_2(x',y',z')$. 
	The unifier of $Q_2$ with $R_1$ that unifies $p_2(x_1,b)$ with $p_2(x',y')$ is internal. The unifier of $Q_2$ with $R_2$ that unifies $\{s_2(a,y_0,z), s_2(x_1,y_1,z)\}$ with $s_2(x',y',z')$ is external because it involves two repeatable patterns.
\end{example}

\subsection{Overview of the Algorithm}

Our query rewriting algorithm takes as input a CQ $Q$ and a set of rules $\R = \R_L \cup \R_T$,
with $\R_L$ a set of linear rules and $\R_T$ a set of transitivity rules,
and produces a finite set of Datalog rules 
and a (possibly infinite) set of CQs.
The main steps of the algorithm are outlined below.
\\[1mm]
	\textbf{Step 1} For each predicate $p$ that appears in $\R_T$, 
	create a pattern definition
	$P := p(\#1,\#2)$, where $P$ is a fresh pattern name. Call the resulting set of 
	definitions $\UP_0$.\\[1mm]
	\textbf{Step 2} Let $\R_L^+$ be the result of considering all of the rule bodies in $\R_L$ and replacing every body atom
	$p(t_1,t_2)$ such that $p$ is a transitive predicate by the repeatable pattern $P^+[t_1,t_2]$.\\[1mm]
	\textbf{Step 3 (Internal rewriting)}
	Initialize $\UP$ to $\UP_0$ and repeat the following operation until fixpoint:
	select a pattern definition $P \in \UP$ and a rule $R \in \R_L^+$ and
	compute the direct rewriting of $\UP$ w.r.t.\ $P$ and $R$. \\[1mm]
	\textbf{Step 4} Replace in $Q$ all atoms $p(t_1,t_2)$ such that $p$ is a transitive predicate
	by the repeatable pattern $P^+[t_1,t_2]$, and denote the result by $\Q^+$. \\[1mm]
	\textbf{Step 5 (External rewriting)} Initialize $\UQ$ to $\{\Q^+\}$ and repeat the following operation until fixpoint:
	choose $\Q_i \in \UQ$, compute a direct rewriting of $\Q_i$ w.r.t.\ $\UP$ and a rule from $\R_L^+$, 
	and add the result to $\UQ$ (except if it is isomorphic to some $\Q_j \in \UQ$).  \\[1mm]
	\textbf{Step 6} Let $\Pi_\UP$ be the Datalog translation of $\UP$, and let $Q_\UQ$ be the (possibly infinite) set of CQs
obtained by 
replacing each repeatable pattern $P^+[t_1,t_2]$ in $\UQ$
	by 
	$p^+(t_1,t_2)$. \\[1.25mm]
The rewriting process in Step~3 is always guaranteed to terminate,
and in Section \ref{sec:prop}, we 
propose a modification to Step 5 that ensures termination
	and formulate sufficient conditions that preserve completeness.
When $\Q_\UQ$ is finite (i.e., it is a UCQ),
 it can be evaluated over the fact base saturated by $\Pi_\UP$, or alternatively,
 translated into a set of Datalog rules, which can be combined with $\Pi_\UP$
and passed to a Datalog engine for evaluation.
Observe 
 that the construction of $\Pi_\UP$ is query-independent 
and can 
be executed as a preprocessing step.

\section{Rewriting Steps in Detail}

A PCQ that contains a repeatable pattern has an infinite number of instances.
Instead of considering all 
instances of a PCQ, we consider a finite set of 
`instances of interest' 
for a given rule.
Such instances will be used for both the internal and external rewriting steps. 
\vspace*{-.4cm}
\paragraph{Instances of interest}
Consider a PCQ $(\Q,\UP)$ and a rule
$R \in \R_L^+$ with head predicate $p$.
	 The {\em instantiations of interest} of $(\Q,\UP)$ w.r.t.\ $R$ are
  	constructed as follows.
	 For each repeatable pattern $P_i^+[t_1,t_2]$ in $\Q$, let
	$a^i_1, \dots, a^i_{n_i}$ be the atoms in the definition of
	$P_i$ with predicate $p$.
	 If $n_i>0$, then expand $P_i^+[t_1,t_2]$ 
	  into $k$ standard patterns,
	where $0 < k \leq min(arity(p),n_i)+2$,  and expand
	each of these standard patterns in turn  
	into some $a^i_\ell$.
	An \emph{ instance of interest} is the instance associated with an instantiation
	of interest.

\begin{example}
\label{ex:interest}
	Reconsider $\Q$, $Q_2$ and $R_2$ from Examples \ref{ex:instantiation} and \ref{ex:instantiation_unifiers}.
	$Q_2$ is not an instance of interest of $\Q$ w.r.t.\ $R_2$ since
	 $P_2[x_1,b]$ is expanded into $p(\#1,\#2)$ whereas the head predicate of $R_2$ is $s_2$.
	If we expand $P_2[x_1,b]$ with $s_2(\#2,y,\#1)$ instead, we obtain the instance of interest
	$Q_3 = s_2(a,y_0,z) \wedge s_2(x_1,y_1,z) \wedge s_2(b,y_2,x_1) \wedge s_1(a,b)$.

\end{example}

We next show that the set of unifiers computed on the
instances of interest of a PCQ `captures' the set of unifiers computed on all of its instances.

\begin{proposition}
\label{prop:inst-interest_unifiers}
	Let $(\Q,\UP)$ be a PCQ and $R \in \R_L^+$.
	For every instance $Q$ of $(\Q,\UP)$ and unifier $\mu$ of $Q$ with $R$,
	 there exist an instance of interest $Q'$ of $(\Q,\UP)$ w.r.t.\ $R$ and a unifier
	$\mu'$ of $Q'$ with $R$ such that $\mu'$ is more general
	\footnote{
		Consider unifiers $\mu=(Q,H,P_{\mu})$ and $\mu'=(Q',H,P_{\mu'})$,
		and let $\sigma$ (resp. $\sigma'$) be a substitution associated
		with $P_{\mu}$ (resp. $P_{\mu'}$).
		We say that
			$\mu'$ is \emph{more general} than $\mu$
		if there is a substitution $h$ from $\sigma'(Q')$ to $\sigma(Q)$ such that $h(\sigma'(Q')) \subseteq \sigma(Q)$ (i.e., $h$ is a homomorphism from $\sigma'(Q')$ to $\sigma(Q)$), and
		for all terms $x$ and $y$ in $Q' \cup H$, if $\sigma'(x) = \sigma'(y)$ then
			$\sigma(h(x)) = \sigma(h(y))$.}
	than $\mu$.
\end{proposition}

\subsection{Internal Rewriting} 
\label{sec:internal-rw}

Rewriting w.r.t.\ internal unifiers
is performed `inside' a repeatable pattern,
independently of the
other patterns and atoms in the query.
We will therefore
handle this kind of rewriting in a query-independent manner
by updating  the
pattern definitions. 

To find all internal unifiers between instances under a repeatable pattern $P^+[t_1,t_2]$ and a
rule head $H = p(\ldots)$, one may think that it is sufficient to consider each
atom $a_i$ in $P$'s definition 
and check if there is an internal unifier of $a_i$ with
$H$.
Indeed, this suffices when predicates are binary: in an internal unifier,
$t_1$ and $t_2$ are unified with distinct variables, which cannot be
existential; thus, the terms in $H$ are 
frontier variables, and a piece 
must consist of a single atom.
If the arity of $p$ is 
greater than 2, the other variables can
be existential,
so it may be possible to unify a path of atoms from $P$'s definition, but not a single such atom 
(see next example). 

\begin{example} \label{ex-internal}
Let $R = s(x,y) \rightarrow r(z_1,x,z_2,y)$ and
$P := r(\#2,\#1,x_0,x_1) | $
$r(\#1,x_2,\#2,x_3) | $
$r(x_4, x_5, \#1, \#2)$.
There is no internal unifier of an atom in $P$'s definition
with $H = r(z_1,x,z_2,y)$. However, if we expand $P^+[t_1,t_2]$ into a path
$P[t_1,y_0] P[y_0,y_1] P[y_1,t_2]$, then expand the $i$th 
pattern of this path into
the $i$th atom in $P$'s definition, 
the resulting instance 
can be unified with $H$ by an
internal unifier (with the partition
$\{\{z_1, y_0, x_4\}, $ $\{x, t_1, x_2, x_5\},$ $ \{z_2,x_0,y_1\},$ $ \{y,x_1,x_3, t_2\}\}$).
\end{example}

Fortunately, we can bound the length of paths to be considered using both the arity of $p$ and
the number of atoms with predicate $p$ in $P$'s definition, allowing us to use 
instances of interest 
 introduced
earlier. 


	A {\em direct rewriting} $\UP'$
	of a set of pattern definitions $\UP$ w.r.t.\ a pattern name $P$ and a rule $R = B \rightarrow H \in \R_L^+$
	is the set of pattern definitions obtained from $\UP$
	by updating $P$'s definition as follows.
	We consider the PCQ $(\Q = P^+[x,y], \UP)$.
	We select an instance of interest $Q$ of $\Q$ w.r.t.\ $R$,
	an internal unifier $\mu$ of $Q$ with $H$,
	and a substitution $\sigma$ associated with $\mu$ that preserves the external terms.
	Let $B'$ be obtained from $\sigma(B)$ by substituting the first (resp.\ second) external term
	by $\#1$ (resp.\ $\#2$). If $B'$ is an atom, we add it to $P$'s definition. Otherwise, $B'$ is a repeatable
	pattern of the form $S^+[\#1, \#2]$ or $S^+[\#2, \#1]$. Let $f$ be a bijection on $\{\#1,\#2\}$:
	if $B'$ is of the form $S^+[\#1, \#2]$, $f$ is the identity, otherwise  $f$ permutes $\#1$ and $\#2$.
	For all $s_i$ in the definition of $S$, we add $f(s_i)$ to $P$'s definition.

	Note that the addition of an atom to a pattern definition is up to \emph{isomorphism}
	(with $\#1$ and $\#2$ 
	treated as distinguished variables,
	i.e., 
	$\#1$ and  $\#2$ are mapped to themselves).
	
\begin{example}
	Reconsider $R$, $\mu$, and the definition of $P$ from Example \ref{ex-internal}.
	Performing a direct rewriting w.r.t.\ $P$ using 
	$R$ and 
	$\mu$
	results in adding the atom $s(\#1,\#2)$ to $P$'s definition.
\end{example}


\begin{proposition}
\label{prop:drw-internal}
	Let $(\Q,\UP)$ be a PCQ where 
	$P^+[t_1,t_2]$ occurs and $R \in \R_L^+$.
	For any instance $Q$ of $(\Q,\UP)$,
	any classical direct rewriting $\Q'$ of $Q$ with $R$
	w.r.t.\ to a unifier internal to $P^+[t_1,t_2]$,
	and any $Q' \in full(\Q',\UP)$,
	there exists a direct rewriting $\UP'$ of $\UP$ w.r.t.\ $P$ and $R$
	such that $(\Q,\UP')$ has a full instance that is
	 isomorphic to $Q'$. 
\end{proposition}

\subsection{External Rewriting} 
\label{sec:external-rw}

	Let $(\Q,\UP)$ be a PCQ, $R \in \R_L^+$,
	$T$ be an instantiation of interest of $(\Q,\UP)$ w.r.t.\ $R$,
	$Q$ be the instance associated with $T$, and
	$\mu = (Q',H,P)$ be an external unifier of $Q$ with $R$.
	From  this, several direct rewritings of $\Q$ w.r.t.\ $\UP$ and $R$
	can be built.
	First, we mark all leaves in $T$ that either have the root as parent
	or are labelled by an atom in $Q'$, 
	 and we restrict $T$ to branches
	leading to a marked leaf.
	Then, we consider each instantiation $T_i$
	that can be obtained from $\Q$
	as follows.
	Replace each repeatable pattern $P^+[t_1,t_2]$ that has
	$k > 0$ children in $T$ by one of the following:
	\begin{itemize}
	\item[(i)] $P^+[t_1,x_1] \wedge X[x_1,x_2] \wedge P^+[x_2,t_2]$,
	\item[(ii)] $P^+[t_1,x_1] \wedge X[x_1,t_2]$,
	\item[(iii)] $X[t_1,x_2] \wedge P^+[x_2,t_2]$,
	\item[(iv)] $X[t_1,t_2]$,
	\end{itemize}	
	where $X[v_1,v_2]$ is a sequence $P[v_1,y_{1}], P[y_1,y_2], \dots,$ $ P[y_{k-1},v_2]$.
	Let $\Q_i$ be the instance associated with $T_i$.

	If $P[x',y']$ in $T$ has child $a(\t)$, 
	expand in $T_i$ the corresponding $P[x,y]$ into $a(\rho(\t))$ where
	$\rho = \{ x' \mapsto x, y' \mapsto y\}$.
	If $\mu' = (\rho(Q'),H,\rho(P))$ is still a unifier of $\Q_i$ with $H$,
	we say that $\Q_i$ is a {\em minimally-unifiable instance} of $\Q$ w.r.t.\ $\mu$.
	In this case,  $\Q'_i = \mu'(Q_i) \setminus \mu'(H) \cup \mu'(B)$ is a
	{\em direct rewriting} of $\Q$ w.r.t.\ $\UP$ and $R$.

\begin{example}
\label{ex:external}
	Reconsider
	$Q_3$ and $R_2$, 
	and  let 
	$\mu = (\{s_2(a,y_0,z),$ $s_2(x_1,y_1,z) \},$ $H_2
	,$ $\{\{a,x_1,x'\},$ $\{y_0,y_1,y'\},$ $\{z,z'\}\})$. 
	First, we consider the instantiation that generated $Q_3$, and we remove the node labelled $P_2[x_1,b]$ and its child
	$s_2(b,y_2,x_1)$, since the latter atom is not involved in $\mu$.
	Next will replace the repeatable pattern $P_1^+[a,z]$ (resp.\ $P_2^+[z,b]$)
	using one of the four cases detailed above,
	and we check whether $\mu'$ (obtained from $\mu$) is still a unifier.
	We obtain in this manner the following minimally-unifiable instances:
	$\Q_1 =$ $P_1^+[a,x_2] \wedge$ $s_2(x_2,y_0,z) \wedge$ $s_2(x_1,y_1,z) \wedge$ $P_2^+[x_1,b] \wedge$ $s_1(a,b)$,
	and
	$\Q_2~=~s_2(a,y_0,z)~\wedge~s_2(x_1,y_1,z)$ $\wedge~P_2^+[x_1,b] \wedge s_1(a,b)$.
	Finally, we rewrite $\Q_1$ and $\Q_2$ into:
	$\Q'_1 = P_1^+[a,x']$ $\wedge$ $s_1(x',y') \wedge$ $P_2^+[x',b] \wedge$ $s_1(a,b)$
	and
	$\Q'_2 = s_1(a,y') \wedge$ $P_2^+[a,b] \wedge$ $s_1(a,b)$.
\end{example}

\begin{proposition}
\label{prop:drw-pcq}
	Let $(\Q,\UP)$ be a PCQ and $R \in \R_L^+$.
	For every $Q \in full(\Q,\UP)$ and every classical direct rewriting $Q'$ of $Q$ with $R$ w.r.t.\ an external unifier,
	there is a 
	 direct rewriting $\Q'$ of $\Q$ w.r.t.\ $\UP$ and $R$ that has an instance
	isomorphic to $Q'$.
\end{proposition}

\section{Termination and Correctness}\label{sec:prop}

To establish the correctness of the query rewriting algorithm, we utilize
Propositions \ref{datalog}, \ref{prop:drw-internal} and \ref{prop:drw-pcq}.

\begin{theorem} \label{thm:correct}
	Let $Q$ be a CQ, $(\F,\R)$ be a {\ltrans} KB, 
	and ($\Pi_{\UP}$,$Q_{\UQ}$) be the (possibly infinite)
	output of the algorithm. 
	Then: $\F,\R \models Q$ iff $\F,\Pi_{\UP} \models Q'$ for some $Q' \in Q_{\UQ}$.
\end{theorem}

Regarding termination, we observe that Step 3 (internal rewriting)
must halt since every direct
rewriting step adds a new atom (using a predicate from $\R_L^+$)
to a pattern definition, and there are
finitely many such atoms, up to isomorphism.

By contrast, Step 5 (external rewriting) need not halt, as the rewritings may
grow unboundedly in size. Thus, to ensure termination, we will modify Step 5 to
exclude 
direct rewritings that increase rewriting size. Specifically, we
identify the following `problematic' minimally-unifiable instances:
\begin{itemize}
\item 
 $Q'$  is composed of atoms expanded from a single pattern $P^+[t_1,t_2]$, $\mu'(t_1) = \mu'(t_2)$, and
$P^+[t_1,t_2]$ is replaced as in case $(i)$, $(ii)$ or $(iii)$.
\item 
$Q'$ is obtained  from the expansion of repeatable patterns, a term $t$ of
$Q$ is unified with an existential variable of the head of the rule, $t$
appears only in repeatable patterns of form $P_i^+[t_i, t]$ (resp.\
$P_i^+[t, t_i]$), and all these repeatable patterns are rewritten as in
case {\it (ii)} $P_i^+[t_i, t'_i] \wedge X[t'_i, t]$  (resp.\ as in case
{\it (iii)} $ X[t, t'_i] \wedge P_i^+[t'_i, t_i]$).
\end{itemize}
We will call a direct rewriting \emph{excluded} if it is based
on such a minimally-unifiable instance; otherwise, it is \emph{non-excluded}.

\begin{example}
The rewriting $\Q'_1$ from Example \ref{ex:external} is excluded
because it is obtained from the minimally-unifiable instance $Q_1$ in which
the repeatable patterns $P_1^+[a,z]$ is expanded as in case $(ii)$
and $P_2^+[z,b]$ as in case $(iii)$, and
$z$ is unified with the existential variable $z'$.
\end{example}

\begin{proposition}\label{excluded}
	Let $(\Q,\UP)$ be a PCQ and $R \in \R_L^+$.
	If $\Q'$ is a non-excluded direct rewriting of $\Q$ with $R$,
	then $|\Q'| \leq |\Q|$.
\end{proposition}

Let us consider the `modified query rewriting algorithm' that is obtained
by only performing non-excluded direct rewritings in Step 5. 
This modification ensures
termination but may comprise completeness.
However, we can show that the modified algorithm is complete 
in the following
key cases:
when the CQ is atomic,
when there is no specialization of a transitive
	predicate,
or when all predicates have arity at most two.
By further analyzing the latter case, we
can formulate a safety condition, defined next, that
guarantees completeness for a much wider class of rule sets.

\vspace*{-3.5mm}
\paragraph{Safe rule sets} 
\label{sec:safety-condition}
We begin by defining a specialization relationship between predicates.
A predicate $q$ is a \emph{direct specialization} of a binary predicate $p$ on positions $\{\iv,\j\}$
$(\iv \neq \emptyset, \j \neq \emptyset)$
if there is a rule of the form $q(\vec{u}) \rightarrow p(x,y)$ 
such that $\iv$ (resp.\ $\j$) contains those positions of $\vec{u}$ that contain the term 
$x$ (resp.\ $y$).
It is a \emph{specialization} of $p$ on positions $\{\iv,\j\}$
if (a) it is a direct specialization of $p$ on positions $\{\iv,\j\}$, or
(b) there is a rule of the form $q(\vec{u})\rightarrow r(\vec{v})$ such that
$r(\vec{v})$ is a specialization of $p$ on positions $\{\k,\l\}$ and
the terms occurring in positions $\{\k,\l\}$ of $\vec{v}$ occur
in positions $\{\iv,\j\}$ of $\vec{u}$ with $\iv \neq \emptyset$ and $\j \neq \emptyset$.
We say that $q$ is a {\em pseudo-transitive predicate} if it is
a specialization of at least one transitive predicate. 

We call a \ltrans\ rule set \emph{safe} if it satisfies the following {\em safety condition}:
for every pseudo-transitive predicate $q$,
there exists a pair of positions $\{i,j\}$ with $i \neq j$ such that
for all transitive predicates $p$ of which $q$ is a specialization on positions $\{\iv,\j\}$,
either $i \in \iv$ and $j \in \j$, or $i \in \j$ and $j \in \iv$.

Note that if we consider binary predicates,
the safety condition is always fulfilled.
Then, specializations correspond exactly to 
the subroles 
considered in DLs.

\begin{example}
\label{ex:safety-condition} 
	Let $R_1 = s_1(x,x,y) \rightarrow p_1(x,y)$,
	$R_2 = s_2(x,y,z) \rightarrow p_2(x,y)$,
	$R_3 = s_1(x,y,z) \rightarrow s_2(z,x,y)$,
	and $p_1$ and $p_2$ be two transitive predicates.

	The following specializations have to be considered:
	$s_1$ is a direct specialization of $p_1$ on positions $\{ \{ 1, 2 \}, \{ 3 \} \}$,
	$s_2$ is a direct specialization of $p_2$ on positions $\{ \{ 1 \}, \{ 2 \} \}$,
	$s_1$ is a specialization of $p_2$ on positions $\{ \{ 3 \}, \{ 1 \} \}$.
	We then have two pseudo-transitive predicates: $s_1$ and $s_2$.
	By choosing the pair $\{1,3\}$ for $s_1$ and $\{1,2\}$ for $s_2$,
	we observe that $\{R_1,R_2,R_3\}$ satisfies the safety condition.

	If we replace $R_3$ by $R_4 = s_1(x,y,z) \rightarrow s_2(x,y,z)$,
	$s_1$ is a specialization of $p_2$ on positions $\{ \{1\},\{2\} \}$,
	and $\{R_1,R_2,R_4\}$ is not safe.
\end{example}

\begin{theorem}
\label{prop:completeness-excluded}
The modified query rewriting algorithm halts.
Moreover, Theorem \ref{thm:correct} (soundness and completeness) holds for the modified 
algorithm if either
the input CQ is atomic, or the input rule set is safe.
\end{theorem}

\section{Complexity}

	A careful analysis of our query rewriting algorithm allows us to
	 provide bounds on the  
	 worst-case complexity of atomic CQ entailment over \ltrans\ KBs,
	 and of general CQ entailment over safe \ltrans\ KBs.
	As usual, we consider two complexity measures:
	\emph{combined complexity} (measured
	in terms of the
	size of the whole input), and
	\emph{data complexity} (measured
	 in terms of the size of the fact base).
	The latter 
	 is often considered more relevant since the fact base
	is typically significantly larger than the rest of the input.

	With regards to data complexity, we show completeness for 
	\NL\ (non-deterministic logarithmic space), which is the same complexity
	as in the presence of transitivity rules alone.
	
\begin{theorem}
\label{thm:data}
	Both (i) atomic CQ entailment over \ltrans\ KBs and (ii) CQ entailment over safe \ltrans\ KBs
	are  \NL-complete in data complexity.
\end{theorem}

Regarding combined complexity, we show that both problems are in ExpTime, 
and prove that  atomic CQ entailment over \ltrans\ KBs is ExpTime-complete.  
Hence, the addition of transitivity rules increases the complexity
of query entailment for atomic queries. The precise combined complexity of 
general CQ entailment over safe \ltrans\ KBs remains an open issue. 

\begin{theorem}
\label{thm:combined}
	Both (i) atomic CQ entailment over \ltrans\ KBs and (ii) CQ entailment over safe \ltrans\ KBs
	are in ExpTime in combined complexity.
	Furthermore, atomic CQ entailment over \ltrans\ KBs is
	ExpTime-hard in combined complexity.
\end{theorem}

\section{Conclusion}

In this paper, we 
made some steps towards a better understanding of the interaction between transitivity and decidable classes of existential rules.
We obtained an undecidability result for aGRD+\fun{trans}, hence for 
\emph{fes}+\fun{trans} and \emph{fus}+\fun{trans}. More positively, 
we established decidability (with the lowest possible data complexity) of atomic CQ entailment over linear+\fun{trans} KBs
and general CQ entailment for safe linear+\fun{trans} rule sets. 
The safety condition was introduced to ensure termination of the rewriting mechanism when predicates of arity more than two are considered (rule sets which use only unary and binary predicates are trivially safe). We believe 
the condition can be removed with a much more involved termination proof.}

In future work, we plan to explore the effect of 
transitivity on \emph{fus} rule classes that are incomparable with linear rules,
namely domain-restricted and sticky rule sets \cite{blms11,cgp10}. 

\section*{Acknowledgements}
This work was supported by ANR project PAGODA (contract ANR 12 JS02 007 01).

\bibliographystyle{named}
\bibliography{bib-ijcai}

\newpage
\section*{Appendix}

\setcounter{proposition}{2}
\setcounter{theorem}{0}
\setcounter{lemma}{0}

\paragraph{Notations}
	In the following we use several notations that were not included
	in the main paper for space restrictions.
	Let $Q_1$ and $Q_2$ be two CQs, if there is a homomorphism from
	$Q_1$ to $Q_2$, we say that $Q_1$ is {\em more general} than $Q_2$,
	and we note $Q_1 \geq Q_2$.
	This definition is naturally extended to (U)PCQs:
	let $(\Q_1,\UP_1)$ and $(\Q_2,\UP_2)$ be two (U)PCQs, if for any full instance $Q_2$
	of $(\Q_2,\UP_2)$ there is a full instance $Q_1$ of $(\Q_1,\UP_1)$ such that
	$Q_1 \geq Q_2$, we say that $(\Q_1,\UP_1)$ is {\em more general} than $(\Q_2,\UP_2)$
	and we note $(\Q_1,\UP_1) \geq (\Q_2,\UP_2)$.

	We recall from the body of the paper that a unifier
	$\mu'=(Q',H,P_{\mu'})$ is more general than
	$\mu=(Q,H,P_{\mu})$ if there is a substitution $h$ from $\sigma'(Q')$ to $\sigma(Q)$ such that $h(\sigma'(Q')) \subseteq \sigma(Q)$ (i.e., $h$ is a homomorphism from $\sigma'(Q')$ to $\sigma(Q)$), and
	for all terms $x$ and $y$ in $Q' \cup H$, if $\sigma_{\mu'}(x) = \sigma_{\mu'}(y)$ then $\sigma_{\mu}(h(x)) = \sigma_{\mu}(h(y))$, where
	$\sigma_{\mu} $ and $\sigma_{\mu'}$ are substitutions associated respectively  with $P_{\mu}$ and $P_{\mu'}$.
	In what follows, we will write $\mu' \geq \mu$ to indicate that $\mu'$ is more general than $\mu$.

\medskip

\noindent\textbf{Proposition \ref{msa}}
\emph{Atomic CQ entailment over \emph{msa}+\fun{trans} KBs is
	undecidable.}
	
\smallskip
\begin{proof}

	The proof is by reduction from atomic
	CQ entailment with general
	existential rules.
	Let $\F$ be a set of facts,
	$\R$ be a set of rules, and $Q$ be an atomic CQ.

	First we consider a new transitive predicate $p$, which 
	is the only transitive predicate we use.

	We next rewrite $\F$ into $\F'$ as follows.
	For each term $t \in terms(\F)$, we add the atoms
	$p(t,a_t)$ and $p(a_t,t)$ to $\F'$, where
	$a_t$ is a fresh constant.

	Then we rewrite $\R$ into a \emph{msa} set of rules $\R^{m}$.
	For each rule $R = B \rightarrow H$, we consider the rule
	$R' = B' \rightarrow H'$ obtained as follows.
	Its body $B'$ is composed of the atoms of $B$
	as well as the atoms $p(t,t)$ for each term
	$t \in terms(B)$.
	Its head $H'$ contains the atoms of $H$ as well as two atoms
	$p(z,x_z)$ and $p(x_z,z)$, where $x_z$ is a fresh variable,
	for each existential variable $z$ in $H$.
	It can be checked that $\R^{m}$ indeed satisfies the MSA property.

	Now, let $R^T$ be the rule expressing the transitivity of $p$.
	It is clear that $(\F,\R) \models Q$ if and only if
	$(\F',\R^m \cup \{R^T\}) \models Q$.

	We conclude that atomic conjunctive query entailment over
	\emph{MSA}+\fun{trans} knowledge bases is undecidable.

\end{proof}

\medskip
\begin{proposition}
\label{prop:Datalog}
	Let $\F$ be a fact base and $(\UQ,\UP)$ be a UPCQ. 
	Then $\F,\Pi_\UP \models Q_\UQ$ iff $\F \models Q$ for
	some $Q \in full(\UQ,\UP)$.
\end{proposition}

\begin{proof} We successively prove the two directions.

$(\Rightarrow)$ Let $T$ be an instantiation of $(\mathbb Q, \mathbb P)$ such
that there exists a homomorphism $\pi$ from its associated full instance  to
the fact base $\mathcal F$. Let us consider a node of $T$ labelled by a
standard pattern atom $P[t_1, t_2]$. The label $r(\rho(\vec t))$ of its
child  was obtained by choosing an atom $r(\vec t)$ in the pattern definition
of $P$. Thus our Datalog program $\Pi_\mathbb P$ contains the rule $r(\vec t) \rightarrow p^+(\#1,
\#2)$ where $\#1, \#2 \in \vec t$. Applying the rule
according to the homomorphism $\pi \circ \rho$, we can add
the atom $p^+(t_1, t_2)$ to $\mathcal F$. Let us repeat this procedure for
every node of $T$ labelled by a standard pattern atom. Consider next a
repeatable pattern atom labelled $P^+[t, t']$ whose children are respectively
labelled $P[t = t_1, t_2], P[t_2, t_3], \ldots, P[t_{k-1}, t_k = t']$.
According to the rule applications already described, $\mathcal F$ now
contains the atoms $p^+(t = t_1, t_2), p^+(t_2, t_3), \ldots, p^+(t_{k-1},
t_k = t')$. Then, by successive applications of the rule in  $\Pi_\mathbb P$ expressing the
transitivity of $p^+$, we finally add to $\mathcal F$ the atom $p^+(t, t')$.
Repeat this procedure for every node of $T$ labelled by a repeatable
pattern atom. Now the root of $T$ is labelled by some $\mathcal Q \in \mathbb
Q$. The UCQ $Q_\mathbb Q$ contains a CQ $Q$ that was obtained from
$\mathcal Q$ by replacing each repeatable pattern $P^+[t_1,t_2]$ by 
	$p^+(t_1,t_2)$. Observe
that the restriction of $\pi$ to the terms of $Q$ is a homomorphism from $Q$
to the fact based obtained from $\mathcal F$ by previous rule applications.

$(\Leftarrow)$ Conversely, let us consider a fact base $\mathcal F'$ obtained
by saturating the initial fact base $\F$ with the rules of $\Pi_\mathbb P$, and a
homomorphism $\pi$ from some CQ $Q'$ in the UCQ $Q_\mathbb Q$ to $\mathcal
F '$. Let us now build an instantiation $T$ whose full instance can be mapped
to $\mathcal F$ thanks to a homomorphism $\pi'$. The root node of $T$ is
labelled by the UPCQ $\mathcal Q$ in $\mathbb Q$ from which $Q'$ was
obtained. Its children are labelled by the atoms and pattern atoms of
$\mathcal Q$. Now we define the restriction of $\pi'$ to the terms of
$\mathcal Q$ as $\pi$. Let us now consider a child of the root labelled by a
repeatable pattern atom $P^+[t, t']$. It follows that $p^+(\pi(t), \pi(t'))$ is
an atom of $\mathcal F '$. Since this atom is not in the initial fact base, it
means that it has been obtained by a (possibly empty) sequence of
applications of the rule expressing the transitivity of $p^+$ on a $p^+$-path
$\pi(t) = t_1, \ldots, t_k = \pi(t')$ such that no atom $p^+(t_i, t_{i+1})$
in $\mathcal F '$ has been obtained by a transitivity rule. Then the node
labelled $P^+[t, t']$ has $k+1$ children respectively labelled $P[t=x_1,
x_2], \ldots, P[x_{k-1}, x_k = t']$. For the fresh variables $x_2, \ldots,
x_{k-1}$, we define $\pi'(x_i) = t_i$.  Repeat this procedure for every
repeatable pattern atom in $T$. Let us next consider a node of $T$ labelled by
a standard pattern atom $P[x, x']$. Since that node was obtained in the
previous phase, we know that the atom $p^+(\pi'(x), \pi'(x'))$ is in
$\mathcal F'$, and that it was not obtained from the application of a
transitivity rule. Thus, the Datalog rule used to produce that atom is
necessarily a rule obtained from the definition of the pattern $P$. Let
$ r(\vec t) \rightarrow p^+(\#1, \#2)$, where $\#1, \#2 \in \vec t$, be that
rule. According to that pattern definition, we can add to the node labelled
$P[x, x']$ a child labelled $r(\rho(\vec t))$. Since the Datalog rule was
applied according to a homomorphism $\pi''$, we define, for every fresh
variable $\rho(t)$, $\pi'(\rho(t)) = \pi''(t)$. Do the same
for every standard pattern atom of $T$. The instance associated
with $T$ is full, and $\pi'$ is a homomorphism from this full instance to the
initial fact base.
\end{proof}

\begin{proposition}
\label{prop:inst-interest_unifiers}
Let $(\Q,\UP)$ be a PCQ and $R \in \R_L^+$.
	For every instance $Q$ of $(\Q,\UP)$ and unifier $\mu$ of $Q$ with $R$,
	 there exist an instance of interest $Q'$ of $(\Q,\UP)$ w.r.t.\ $R$ and a unifier
	$\mu'$ of $Q'$ with $R$ such that $\mu'$ is more general than $\mu$.
\end{proposition}

\begin{proof}
	Let $(\Q,\UP)$ be a PCQ, $R \in \R_L^+$ be a rule with head $p(\t)$,
	$Q$ be an instance of $(\Q,\UP)$, and
	$\mu = (Q_2,p(\t),P_u)$ be a unifier of $Q$ with $R$.

	Consider a repeatable pattern $P^+[t_1,t_2]$  from which some atoms
	in $Q_2$ are expanded.
	If $P^+[t_1,t_2]$ is expanded into $k \leq arity(p) + 2$ standard
	patterns in $Q$\footnote{Strictly speaking, we mean the instantiation underlying $Q$,
	but to simplify the notation, here and later in the appendix, we will often refer to instances, leaving the instantiation implicit.},
	then there is an instance of interest $Q'$ that expands $P^+[t_1,t_2]$
	into exactly $k$ standard patterns. Thus, there is an
	isomorphism $\pi$ between the atoms expanded under $P^+[t_1,t_2]$
	in $Q$ and in $Q'$.
	Assume instead that $P^+[t_1,t_2]$ is expanded in $Q$ into $k > arity(p)+2$
	standard patterns.
	We denote by $\sigma$ a substitution associated with $P_u$, and by
	$P[t_1=x_0,x_1], P[x_1,x_2], \dots, P[x_{k-2},x_{k-1}], P[x_{k-1},x_k=t_2]$
	the sequence of standard patterns expanded from $P^+[t_1,t_2]$ in $Q$,
	and let $x_s$ and $x_e$ ($s < e$) be the external terms of $P^+[t_1,t_2]$
	w.r.t.\ $\mu$.
	The unifier is single-piece (cf.\ Section \ref{sec:prelims}), thus,
	for every $0 < i < k$, $\sigma(x_i) = \sigma(z_i)$ for some
    existential variable $z_i$ from the head of $R$.

	We construct an instance $Q'$ and function $\pi$ as follows.
	Starting from $\Q$, we expand every repeatable pattern $P^+[t_1,t_2]$ that is relevant for $\mu$
	into $e-s$ standard patterns
	(where $e$ and $s$ are defined as above and depend on the particular pattern):
	$$P[t_1= x'_s,x'_{s+1}], P[x'_{s+1},x'_{s+2}], \dots, P[x'_{e-1},x'_{e}=t_2].$$
	Then for every $s \leq i \leq e$,  we set $\pi(x'_i) = x_{i}$, and we expand
	$P[x'_i,x'_{i+1}]$ into the atom $a_{i}'$ that is obtained from the
	the atom $a_i$ expanded under $P[x_i,x_{i+1}]$ in 
	$Q$ by replacing every $x_j$ with $x_j'$.
	If $e-s \leq arity(p)+2$, we are done.
	Otherwise, we will need to remove some patterns
	in order to satisfy the definition of instances of interest. To this end,
	we define a sequence $s < i_1 < j_1 < \ldots < i_m < j_m < e$ of indices as follows:
	\begin{itemize}
	\item We call $i < j$, with $s < i < j < e$,  a \emph{matching pair} if $a_i$ and in $a_j$,
	$x'_{i+1}$ and $x'_{j}$ occur at the same position of $p$ (hence, $\sigma(x_{i+1})=\sigma(x_j)$);
	\item We say that a matching pair $i < j$ is \emph{maximal w.r.t.\   index $\ell$} if the following conditions hold:
	\begin{itemize}
	\item $i \geq \ell$,
	\item there is no matching pair $i' < j'$ with $\ell \leq  i' < i$
	\item there is no matching pair $i < j'$ with $j' > j$
	\end{itemize}
	\item We let $i_1 <j_1$ be the matching pair that is maximal w.r.t.\  index $s+1$
	\item If $i_k < j_k$ is already defined, then we let $i_{k+1} < j_{k+1}$ be the matching pair
	that is maximal w.r.t.\ index $j_k$, if such a pair exists (otherwise, $i_k < j_k$ is the final pair in the sequence).
	\end{itemize}
	Now remove from $Q'$ all of the patterns $P[x'_\ell,x'_{\ell+1}]$ such that $i_g < \ell < j_g$ for some $1 \leq g \leq m$,
	as well as the atoms that are expanded from such patterns.
	We claim that there are now at most $arity(p)+2$ patterns $P[x'_\ell,x'_{\ell+1}]$ below $P^+[t_1,t_2]$ in $Q'$.
	Indeed, if this were not the case, we could find a matching pair $i < j$ among the remaining patterns.
	Since $i_1 < j_1$ is maximal w.r.t.\  index $s+1$, and there are no further matching pairs starting from $j_m$,
	we know that $i \geq i_1$ and $i < j_m$. Moreover, since $a_i$ is still present in $Q'$,
	it must be the case that $j_g < i < i_{g+1} $ for some $1 \leq g < m$. But this contradicts the fact that $i_{g+1} < j_{g+1}$
	is maximal w.r.t.\  $j_g$.
	
	In order for the different remaining patterns to form a sequence, we will need to
	perform a renaming of terms. If there are $n$ patterns left under $P^+[t_1,t_2]$, then we will rename these
	patterns from left to right by:
	$$P[t_1= x''_0,x''_{1}], P[x''_1,x''_{2}], \dots, P[x''_{n-1},x''_{n}=t_2].$$
	and will rename the atoms underneath these patterns accordingly.
	We will also update $\pi$ by setting $\pi(x''_i)=\pi(x'_j)$ if $x'_j$ was renamed into $x''_i$
	and there is no $x'_{j'}$ with $j' < j$ that was also renamed into $x''_i$.

	Let $Q'$ be the instance obtained in this manner.
	We note that by construction, it is an instance of interest of $(\Q,\UP)$ w.r.t.\ $R$, as we only
	expand patterns into atoms that use the predicate $p$ from the rule head, and the number of patterns
	generated from any repeatable pattern is at most $arity(p)+2$.

	Regarding $\pi$, note that a term may be shared among several repeatable patterns that are relevant for $\mu$.
	However, we can show that if a term is shared by multiple relevant patterns,
	then the (partial) mapping associated with those patterns will agree on the shared
	term, i.e.\ $\pi$ is well defined. First note if a term is shared by two repeatable
	patterns, then it must appear as one of the distinguished terms ($t_1, t_2$) in both
	patterns. Moreover, by tracing the above construction, we find that $\pi$ is the identity
	on such terms.
	
	To complete the definition of $\pi$, we extend it to all of the terms of $Q'$ by letting $\pi$
	be the identity on all terms that do not occur underneath a developed repeatable pattern (i.e.,\ terms that appear in
	a repeatable pattern that is not expanded, or in one of the standard atom of $\Q$).
	Observe that $\pi$ is an injective function, so its inverse $\pi^{-1}$ is well-defined.

	Now let $Q_2'$ consist of all atoms in $Q' \cap Q_2$ that are not expanded from a repeatable pattern (i.e., they are standard
	atoms from $\Q$) as well as all atoms in $Q'$ that lie under a repeatable pattern.

	Note that by construction every term $t$ in $Q_2'$ is such that $\pi(t)$ appears in $Q_2$.
	We can thus define a partition $P'_u$ of the terms in $Q'_2 \cup H$ by taking every
	class $C$ in $P_u$ and replacing every term $t$ from $Q_2$ by $\pi^{-1}(t)$, if such a term exists,
	and otherwise deleting $t$; terms from $p(\vec{t})$ are left untouched.
	Moreover, by the injectivity of $\pi$,
	every term appears in at most one class, i.e., $P'_u$ is indeed a partition.

	We aim to show that  $\mu' = (Q_2',p(\t),P'_u)$ is the desired unifier,
	We first show that $\mu'$ is a unifier of $Q'$ with $R$. In what follows,
	it will prove convenient to extend $\pi$ to the terms in the head atom $p(\t)$,
	by letting $\pi$ be the identity on such terms.
	We will let $\sigma$ be a substitution associated with $\mu$, and let $\sigma'$ be the corresponding
	substitution for $\mu'$ defined by setting $\sigma'(t)=\sigma(\pi(t))$.
	\begin{itemize}
		\item $P'_u$ is admissible: since $\pi$ is the identity on constants,
		if a class in $P'_u$ contains two constants $c,d$,
		then the corresponding class in $P_u$ must also contain $c,d$ (a contradiction).
		\item $\sigma'(p(\t)) = \sigma'(Q'_2)$: since $\sigma'(p(\t)) =\sigma(p(\t))$ (due to our choice of $\sigma'$),
		it suffices to prove that $\sigma'(Q'_2) \subseteq \sigma(Q_2)$.
	First take some atom $\alpha$
	that belongs to $Q'_2 \cap Q_2$. Then we have $\pi(\alpha)=\alpha$,
	so $\sigma'(\alpha)=\sigma(\pi(\alpha)) \in \sigma(Q_2)$. 	
	Next consider the case of an atom $\alpha$ that belongs to $Q_2$ but not $Q'_2$.
	Then $\alpha$ must lie below a repeatable pattern $P^+[t_1,t_2]$ that is expanded into $k > arity(p)+2$
	standard patterns $P[t_1=x_0,x_1], P[x_1,x_2], \dots, P[x_{k-2},x_{k-1}], P[x_{k-1},x_k=t_2]$ in $Q$.
	In this case, $P^+[t_1,t_2]$ is expanded in $Q'$ into
	$$P[t_1= x'_s,x'_{s+1}], P[x'_{s+1},x'_{s+2}], \dots, P[x'_{e-1},x'_{e}=t_2],$$
	and each $P[x'_i,x'_{i+1}]$ is expanded into $a_{i}'$. If $e-s \leq arity(p)+2$,
	then the atoms $a'_i$ all belong to $Q'_2$. If we have $\alpha = a'_i$,
	then we have $\pi(a'_i) = a_i$, hence $\sigma'(\alpha)=\sigma(\pi(a'_i))=\sigma(a_i) \in \sigma(Q_2)$.
	The final possibility is that $e-s > arity(p)+2$, in which case some of the patterns
	will be removed and the remaining patterns will be renamed (as will be their corresponding atoms).
	Suppose that $\alpha$ is the atom $a''_h$ below the pattern $P[x''_\ell,x''_{\ell+1}]$,
	which was obtained from renaming the pattern $P[x_h,x_{h+1}]$.
	We claim that $\sigma'(\alpha)=\sigma(\pi(\alpha)) = \sigma(a_h)$, hence $\sigma'(\alpha) \in \sigma(Q_2)$.
	By examining the way renaming is performed, there are two situations that can occur:
	\begin{itemize}
	\item $\pi(x''_\ell)=x_h$ and $\pi(x''_{\ell+1})=x_{h+1}$: in this case, $\pi(a''_h)= a_h$, hence $\sigma'(\alpha)= \sigma(a_h)$.
	\item $\pi(x''_\ell) \neq x_h$: in this case, there must exist a matching pair $i_g < j_g$ such that
	$h=j_g$, $\pi(x''_\ell)=x_{i_g+1}$, and $\pi(x''_{\ell+1})=x_{h+1}$.
	From the definition of matching pairs, we know that $\sigma(x_{i_g+1}) = \sigma(x_{j_g})$. 
	It follows that $\sigma'(x''_\ell)= \sigma(\pi(x''_\ell))= \sigma(x_{i_g+1})= \sigma(x_{h})$
	and $\sigma'(x''_{\ell+1})= \sigma(\pi(x''_{\ell+1}))= \sigma(x_{h+1})$.
	We can thus conclude that $\sigma'(\alpha)= \sigma(a_h)$.
	\end{itemize}

	%

		\item for a contradiction, suppose the class $C'$ in $P'_u$ contains an existential variable $z$ from $H$
		and either a constant or a variable that occurs in $Q' \setminus Q'_2$.
		If it contains a constant $c$, then the corresponding class $C$ in $P_u$ will contain both $z$ and $c$,
		 i.e., $C$ is not a valid class.
		Next suppose that $C'$ contains a variable $x$ that occurs in $Q' \setminus Q'_2$,
		which means that the corresponding class $C$ in $P_u$ contains $\pi(x)$.
		Since $x$ that occurs in $Q' \setminus Q'_2$, it must either appear in a standard atom of $Q'$ that
		does not appear under any repeatable pattern
		or in a repeatable pattern that is not developed in $Q'$.
		In the former case, the same atom appears in $Q \setminus Q_2$, and in the latter case,
		since $Q$ is full, there is an atom in $Q$ that is developed from the repeatable pattern and contains $\pi(x)$,
		but which does not participate in $Q_2$. In both cases, we obtain a contradiction.
	\end{itemize}

	Finally, we show that $\mu'$ is more general than $\mu$:
	\begin{itemize}
	\item $\sigma'(Q'_2)) \subseteq \sigma(Q_2)$: proven above.
	\item if $\sigma'(u_1)=\sigma'(u_2)$, then $u_1, u_2$ belong to the same class in $P'_u$,
	and so $\pi(u_1)$ and $\pi(u_2)$ must belong to the same class in $P_u$.
	\end{itemize}
	We have thus shown that $Q'$ is an instance of interest of $(\Q,\UP)$ w.r.t.\ $R$
	such that there is a unifier $\mu'$ of $Q'$ with $R$ with $\mu' \geq \mu$.
\end{proof}

\medskip

\begin{proposition}
\label{prop:drw-internal}
	Let $(\Q,\UP)$ be a PCQ where 
	$P^+[t_1,t_2]$ occurs and $R \in \R_L^+$.
	For any instance $Q$ of $(\Q,\UP)$,
	any classical direct rewriting $\Q'$ of $Q$ with $R$
	w.r.t.\ to a unifier internal to $P^+[t_1,t_2]$,
	and any $Q' \in full(\Q',\UP)$,
	there exists a direct rewriting $\UP'$ of $\UP$ w.r.t.\ $P$ and $R$
	such that $(\Q,\UP')$ has a full instance that is
	 isomorphic to $Q'$.
\end{proposition}

\begin{proof}
	Let $(\Q,\UP)$ be a PCQ where $P^+[t_1,t_2]$ occurs,
	$R = (B \rightarrow H) \in R_L^+$, $Q$ be an instance of $(\Q,\UP)$,
	$\mu = (Q_2,H,P_u)$ be a unifier internal to $P^+[t_1,t_2]$ of $Q$ with $R$,
	$\Q'$ be the classical direct rewriting of $Q$ with $R$ w.r.t.\ $\mu$,
	and $Q'$ be a full instance of $(\Q',\UP)$.

	Since $\mu$ is internal to $P^+[t_1,t_2]$, all atoms in $Q_2$ are expanded
	from $P^+[t_1,t_2]$ in $Q$, and do not unify $t_1$ with $t_2$, nor $t_1$ (resp.\ $t_2$)
	with an existential variable from $H$.
	We denote by $P[t_1=x_0,x_1], P[x_1,x_2], \dots, P[x_{k-1},x_k = t_2]$ the sequence
	of standard patterns expanded under $P^+[t_1,t_2]$ in $Q$,
	$x_s$ and $x_e$ ($s < e$) the external terms of $P^+[t_1,t_2]$ w.r.t.\ $\mu$,
	and $a_i$ the atom expanded under $P[x_i,x_{i+1}]$.
	From Prop. \ref{prop:inst-interest_unifiers}, there is a unifier $\mu'$ of an instance of interest $Q_3$ of $\Q$
	with $R$ with $\mu' \geq \mu$.
	Since $x_s$ and $x_e$ are not unified with existential variables,
	let $Q_4$ be the CQ obtained from $Q_3$ by removing all atoms and patterns that are not
	relevant for $\mu'$.
	Obviously $Q_4$ is an instance of interest of a PCQ of form $P^+[t_1,t_2]$.
	Let $\UP'$ be the direct rewriting of $\UP$ w.r.t.\ $\mu'$, obtained from $Q_4$.

	Let $A_{l} = \{P[x_i,x_{i+1}]~|\ 0 \leq i < s\}$,
	$A_m = \{P[x_i,x_{i+1}]~|\ s \leq i < e\}$,
	$A_{r} = \{P[x_i,x_{i+1}]~|\ e \leq i < k\}$,
	and $A = A_l \cup A_m \cup A_r$. Further let
	$A'_{l}$ (resp.\ $A'_m$, $A'_r$, $A'$)  be the set of atoms expanded under $A_{l}$
	(resp.\ $A_m$, $A_r$, $A$) in $Q$.

	Initialize $Q''$ to $Q \setminus A' \cup \{P^+[t_1,t_2]\}$.
	One can see that $Q''$ is an instance of both $(\Q,\UP)$ and $(\Q,\UP')$.
	If $B$ is a not a repeatable pattern, let $\ell = 1$, otherwise let $S$ be the repeatable pattern
	in $B$, and
	$S[x'_0,x'_1], \dots, S[x'_{\ell-1},x'_{\ell}]$ be the sequence expanded from
	$S^+[x'_0,x'_{\ell}]$ in $Q'$.
	We denote by $a'_i$ the atom expanded under $S[x'_i,x'_{i+1}]$.
	Then expand $P^+[t_1,t_2]$ in $Q''$ into
	$k' = |A_l| + |A_r| + \ell$ standard patterns:
	$P[t_1=x''_0,x''_1], \dots, P[x''_{k-1},x''_{k'} = t_2]$.
	Let $\pi$ be the function defined as follows:
	\begin{itemize}
	\item for all $0 \leq i \leq s$, $\pi(x''_i)= x_i$;
	\item for all $s < i < s + \ell$, $\pi(x_i'')=x_{i-s}'$;
	\item for all $s+ \ell \leq i \leq k'$, $\pi(x''_i)= x_{i- \ell + (e-s)}$.
	\end{itemize}
	Note that $\pi$ is injective, so its inverse exists.
	Expand all $P[x''_i,x''_{i+1}]$
	with $0 \leq i < s$ or $s + \ell \leq i < k$
	(resp. $s \leq i < s+ \ell$)
	into $\pi^{-1}(a_i)$ (resp. $\pi^{-1}(a'_i)$).
	Finally, for all terms $u$ in $Q''$ for which $\pi$ is not defined (i.e., those terms
	appearing in atoms that were not expanded from the pattern $P^+[t_1,t_2]$),
	we set $\pi(u)=u$.

	By construction, $Q''$ is still an instance of $(\Q,\UP')$
	and $\pi$ is an isomorphism between $Q'$ and $Q''$.
\end{proof}

\medskip

\begin{proposition}
\label{prop:drw-pcq}
Let $(\Q,\UP)$ be a PCQ and $R \in \R_L^+$.
	For every $Q \in full(\Q,\UP)$ and every classical direct rewriting $Q'$ of $Q$ with $R$ w.r.t.\ an external unifier,
	there is a 
	 direct rewriting $\Q'$ of $\Q$ w.r.t.\ $\UP$ and $R$ that has an instance
	isomorphic to $Q'$.

\end{proposition}

\begin{proof}
	Let $(\Q,\UP)$ be a PCQ, $R = (B \rightarrow H) \in R_L^+$,
	$Q \in full(\Q,\UP)$, $\mu = (Q_u,H,P_u)$ be an external unifier of $Q$ with $R$,
	and $Q'$ be the classical direct rewriting of $Q$ with
	$R$ w.r.t.~$\mu$.

	From Proposition \ref{prop:inst-interest_unifiers},
	there is an instance of interest $Q_2$ of $(\Q,\UP)$
	such that there is a unifier $\mu' = (Q_{u'},H,P_{u'}) \geq \mu$ of $Q_2$ with $R$.
	We denote by $\sigma$ (resp.\ $\sigma'$) a substitution
	associated with $\mu$ (resp.\ $\mu'$).

	For any repeatable pattern $P^+[t_1,t_2]$ in $\Q$,
	build $A, A_l, A_m$ and $A_r$ as in the proof of Proposition\ \ref{prop:drw-internal} using
	the instance $Q_2$ and unifier $\mu'$.
	Assume $t_1$ (or $t_2$) is unified with an existential variable, then
	from the condition on external unifiers, either $A_l$ or $A_r$ is empty.
	Consider the minimally-unifiable instance $\Q_M$ of $\Q$ w.r.t.\ $\mu'$ that
	replaces $P^+[t_1,t_2]$ by:
	(i) $A_m$ if $A_l = A_r = \emptyset$;
	(ii) $P^+[t_1,x_s], A_m$ if $A_r = \emptyset$ and $A_l \neq \emptyset$;
	or (iii) $A_m, P^+[x_e,t_2]$ if $A_l = \emptyset$ and $A_r \neq \emptyset$.
	In case (ii) (resp. (iii)), since all atoms in $A_l$ (resp. $A_r$)
	are not involved in $\mu'$,
	$x_s$ (resp. $x_e$) is not unified with an existential variable
	(or the piece condition on unifiers would not be satisfied).
	Therefore, $\mu'$ is a unifier of $\Q_M$ with $R$.
	We let $\Q'$ be the direct rewriting of $\Q_M$ w.r.t.\ $\mu'$ and $R$.

	Note that each repeatable pattern $P^+[t_1,t_2]$ in $Q'$ expands into
	$A_l \wedge \sigma(B) \wedge A_r$, and
	in $\Q'$ there is a $P^+[t_1,x_s]$ (resp.\ $P^+[x_e,t_2]$) iff
	$A_l$ (resp.\ $A_r$) is not empty.
	Thus consider $Q''$ obtained from $\Q'$ by expanding $P^+[t_1,x_s]$ (resp.\ $P^+[x_e,t_2]$)
	into $k$ standard patterns
	where $k = |A_l|$ (resp.\ $k = |A_r|$), and choose the same atoms as in $A'_l$ (resp.\ $A'_r$).
	Since $\mu' \geq \mu$, there is an
	homomorphism $\pi$ from $\sigma'(Q_{u'})$ to $\sigma(Q_u)$.
	Note that if we restrict $\pi$ to terms
	in $\sigma'(B)$, $\pi$ is an isomorphism.
	Furthermore, we can extend $\pi$ to $Q''$ in the same
	way as we did in the previous proof.
	Thus $Q''$ is isomorphic to $Q'$.
\end{proof}

\medskip

\begin{proposition}
\label{prop:size}
	Let $(\Q,\UP)$ be a PCQ and $R \in \R_L^+$.
	If $\Q'$ is a non-excluded direct rewriting of $\Q$ with $R$,
	then $|\Q'| \leq |\Q|$.
\end{proposition}

\begin{proof}
	Let $(\Q,\UP)$ be a PCQ, $R = (B \rightarrow H) \in R_L^+$, $Q$ be a non-excluded minimally-unifiable instance of
	$(\Q,\UP)$, $\mu = (Q',H,P_u)$ be an external unifier of $Q$ with $R$, and $\sigma$ the substitution induced by $P_u$.

	Note that all repeatable patterns $P^+[t_1,t_2]$ are at most replaced by the sequence $S$ needed
	by the unifier (i.e., $S \subseteq Q'$), plus a single repeatable pattern $P^+[t_1,x_1]$ (or $P^+[x_k,t_2]$).
	Indeed, the only situation that would lead us to introduce more than one more repeatable
	pattern (i.e., as in External Rewriting case $(i)$) is when
	either $t_1$ or $t_2$ is unified with an existential variable.
	However, if $t_1$ (or $t_2$) is unified with an
	existential variable, because of the piece condition on
	unifiers, no unifier of $P^+[t_1,x_1] \wedge S \wedge
	P^+[x_k,t_2]$ can be found.

	Since $|B| = 1$, we have to show that all atoms that were introduced when replacing a repeatable pattern
	are erased by the direct rewriting of $Q$ w.r.t.\ $\mu$.

	If $Q'$ consists of at least one atom that is not expanded from a pattern, the direct rewriting
	of $Q$ w.r.t.\ $\mu$ erases this atom.

	Next assume $Q'$ consists only of atoms expanded from repeatable patterns.
	If $Q' = \{P^+[t_1,t_2]\}$ and neither $t_1$ nor $t_2$ is unified
	with an existential variable, then $\sigma(t_1) = \sigma(t_2)$, so
	the only non-excluded minimally-unifiable instance of $\Q$ w.r.t.\ $\mu$
	replaces $P^+[t_1,t_2]$ only by the sequence $S$ needed by the unifier
	(see the first condition on non-excluded minimally-unifiable instances).
	Thus, the direct rewriting erases $P^+[t_1,t_2]$.

	Otherwise,
	we know that at least one $P^+[t_1,t_2]$ from $\Q$ is replaced
	by the sequence $S$ involved in the unifier
	(see the second condition on non-excluded minimally-unifiable instances),
	thus there is at least one $P^+[t_1,t_2]$ erased by the direct rewriting.
\end{proof}

\medskip

We will break the proof of Theorem \ref{thm:correct} into the following five lemmas.

\begin{lemma}
\label{prop:rw-trans}
	Let $Q$ be a CQ, $R \in R^T$, $\UP_0$ be the initial set of pattern definitions
	relative to $R^T$ (see Step 1 of the algorithm overview), and $\Q^+$ be obtained from $Q$ by replacing all atoms $p(t_1,t_2)$
	such that $p$ is a transitive predicate by $P^+[t_1,t_2]$.
	If there is a classical direct rewriting $Q'$ of $Q$ with $R$, then
	there is a full instance $Q''$ of $(\Q^+, \UP_0)$ that is isomorphic to $Q'$.
\end{lemma}

\begin{proof}
	Let $p(t_1,t_2)$ be the atom of $Q$ that is rewritten to obtain $Q'$.
	Since $p$ is a transitive predicate, it occurs in a pattern definition $P_0$ in $\UP$,
	and $\Q^+$ contains the atom $P^+[t_1,t_2]$.
	In $Q'$, $p(t_1,t_2)$ is rewritten into $p(t_1,x_1) \wedge p(x_1,t_2)$.
	Let $Q''$ be the full instance of $(\Q^+,\UP_0)$ that expands all repeatable patterns
	but $P^+[t_1,t_2]$ into a single standard pattern,
	expands $P^+[t_1,t_2]$ into two standard patterns $P[t_1,x'_1], P[x'_1,t_2]$,
	and then further expands the standard patterns using the unique atom in each of the pattern definitions.
	It is clear that $Q''$ is isomorphic to $Q'$ (simply map $x'_1$ to $x_1$ and all other terms to themselves).
\end{proof}

\begin{lemma}
\label{lemma:instances-plus}
	Let $\UP$ be a set of pattern definitions, $\UP_0
	\subseteq \UP$ be the initial set of patterns definitions
	built from the set $\R_T$ of transitivity rules, $(\Q,\UP)$ be a PCQ that
	does not contain any standard atom using a transitive predicate,
	$Q$ be a full
	instance of $(\Q,\UP)$, and $\Q^+$ be obtained from $Q$
	by replacing all atoms $p(t_1,t_2)$ with $p$ transitive  
	by $P^+[t_1,t_2]$.
	
	Then, for every full instance $Q'$ of $(\Q^+,\UP_0)$, there is
	a full instance $Q''$ of $(\Q,\UP)$ such that
	$Q''$ is isomorphic to $Q'$.
\end{lemma}

\begin{proof}
	We build the instance $Q''$ as follows.
	Initialize $Q''$ to the atoms and repeatable patterns occurring in $\Q$.
	Next, for all repeatable patterns $P_i^+[t_1,t_2]$ in the instantiation underlying $Q$
	consider each of the atom that is expanded from a child of $P_i^+[t_1,t_2]$ in turn, working from left to right.
	If the atom $p(\t)$ under $P_i[u,v]$ is being considered, 
	then do the following:
	\begin{itemize}
	\item if $p$ is not a transitive predicate, then add a single child $P_i[u,v]$ to $P_i^+[t_1,t_2]$, and expand it into $p(\t)$.
	\item if $p$ is transitive, then $p(\t)$ has been replaced in
	$\Q^+$ by $P^+[\t]$. We also know that $\t$ consists of the terms $u,v$ from $P_i[u,v]$.
	We suppose that $p(\t)=p(u,v)$ (hence $P^+[\t]=P^+[u,v]$); a similar argument can be used if the positions are reversed.  
	Let $P[u= x_0,x_1],\dots,P[x_{k-1},x_k=v]$ be the children of $P^+[u,v]$ in $Q'$, and
	$a_\ell$ be the atom expanded under $P[x_{\ell},x_{\ell+1}]$ ($0 \leq \ell < k$).
	In place of the child $P_i^+[u,v]$ in $\Q^+$, we will add $k$ children to $P^+_i[t_1,t_2]$ in $Q''$:
	$P_i[u= x_0,x_1],\dots,P_i[x_{k-1},x_k=v]$, and expand
	$P_i[x_j,x_{j+1}]$  into $a_{j}$. Note that we may assume that the terms $x_i$ ($0 < i < k$) are fresh, i.e.,
	they do not already appear in $Q''$.
	\end{itemize}
	It can be verified that the resulting full instance $Q''$ is isomorphic to $Q'$. Indeed, all atoms in $Q'$
	that are also in $Q$ are present in $Q''$. All other atoms belong to a sequence of transitive atoms,
	which we have reproduced (modulo renaming of variables) in $Q''$.
\end{proof}

\begin{lemma}
\label{prop:rw-eq-prw}
	Let $Q$ be a CQ and $\R$ be a set of \ltrans\ rules, and let $(\UQ,\UP)$ be the output of the algorithm.
	For any $Q'$ obtained from a sequence of classical direct rewritings of $Q$
	with $\R$,
	there is a PCQ $(\Q,\UP)$ with $\Q \in \UQ$ and a full instance $Q''$ of $(\Q,\UP)$ s.t.\ $Q''$ is isomorphic to $Q'$.
\end{lemma}

\begin{proof}
	Let $Q = Q_0, \mu_1, Q_1, \mu_2, Q_2, \dots, \mu_k, Q_k = Q'$ be a sequence of classical direct
	rewritings from $Q$ to $Q'$, and let $R_1, \ldots, R_k$ be the associated sequence of rules from $\R$.
	
	We show the desired property, by induction on $0 \leq i \leq k$. 	
	For the base case ($i=0$), we can set $\Q_0=Q^+$, since $Q_0=Q$ is clearly a full instance
	of $(\Q_0,\UP)$.
	
	For the induction step, suppose that we have $\Q_{i-1} \in \UQ$ and a full instance $Q''_{i-1}$ of $(\Q_{i-1}, \UP)$
	that is isomorphic to the CQ $Q_{i-1}$. There are two cases to consider, depending on the type of the rule $R_i$.

	If $R_i$ is a transitivity rule, then from Lemma
	\ref{prop:rw-trans}, 
	 we know that $\Q^+_{i-1}$ (obtained from $Q_{i-1}$ by replacing every transitive predicate $p$ by pattern $P^+$)
	is such that there is a full instance $Q_{i-1}^{+}$ of $(\Q^+_{i-1}, \UP)$ that is isomorphic to $Q_i$.
	Furthermore, we know that $\Q_{i-1}$ cannot contain any standard atoms with transitive predicates, since every
	PCQ produced in Step 5 contains patterns for the transitive predicates. Thus, we may apply Lemma \ref{lemma:instances-plus}
	and infer that $Q_{i-1}^{+}$ is isomorphic to some full instance
	$Q_{i-1}''$ of $(\Q_{i-1},\UP)$.
	Therefore, $Q_{i-1}''$ is isomorphic to $Q_{i}$.

	If $R_i$ is not a transitive rule,
	since $Q_{i-1}$ is isomorphic to some full instance $Q''_{i-1}$ of $(\Q_{i-1}, \UP)$,
	let $\mu_i''$ be the unifier of $Q''_{i-1}$ with $R_i$ obtained from $\mu$
	and the isomorphism between $Q_{i-1}$ and $Q''_{i-1}$.
	If $\mu_i''$ is internal to some repeatable pattern, then from Proposition \ref{prop:drw-internal},
	we know that there is an instance $Q'_{i}$ of $(\Q_{i-1},\UP)$ that is isomorphic to $Q_i$.
	Otherwise, from Proposition \ref{prop:drw-pcq}, there exists $\mu_i'$ and
	a direct rewriting $\Q_i$ of $\Q_{i-1}$ with $\mu'_i$ such that
	there is an instance $Q'_{i}$ of $(\Q_i,\UP)$ that is isomorphic to $Q_i$.

	We have thus completed the inductive argument and can conclude that
	there is a PCQ $(\Q,\UP)$ with $\Q \in \UQ$ and a full instance $Q''$ of $(\Q,\UP)$ s.t.\ $Q''$ is isomorphic to $Q' = Q_k$.
\end{proof}

\begin{lemma}
\label{prop:complete}
	Let $Q$ be a CQ, $(\F,\R)$ be a {\ltrans} KB,
	and ($\Pi_{\UP}$,$Q_{\UQ}$) be the output of the algorithm.
	If $\F,\R \models Q$ then $\F,\Pi_{\UP} \models Q'$ for some $Q' \in Q_{\UQ}$.
\end{lemma}

\begin{proof}
	Since $\F,\R \models Q$, there is a (finite) classical rewriting $Q'$ of $Q$ with $\R$ such that
	$\F \models Q'$.
	From Proposition \ref{prop:rw-eq-prw}, there is
	there is a PCQ $(\Q,\UP)$ with $\Q \in \UQ$ and a full instance $Q''$ of $(\Q,\UP)$ s.t.\ $Q''$ is isomorphic to $Q'$.
	Therefore, $\F \models Q''$.
	We conclude by Proposition \ref{prop:Datalog}.
\end{proof}

\begin{lemma}
\label{prop:sound}
	Let $Q$ be a CQ, $(\F,\R)$ be a {\ltrans} KB, 
	and ($\Pi_{\UP}$,$Q_{\UQ}$) be the output of the algorithm. 
	If $\F,\Pi_{\UP} \models Q_{\UQ}$ then $\F,\R \models Q$.
\end{lemma}

\begin{proof}
Let $\UP$ be the set of pattern
definitions computed in Step 3  of the
algorithm, and $\Pi_\UP$ the corresponding set of Datalog rules.
Consider the CQ $Q^{++}$obtained from $Q$ by replacing every atom $p(t_1,t_2)$
such that $p$ is transitive by the atom $p^+(t_1,t_2)$.
The following claim establishes the soundness of the
internal rewriting mechanism in Step 3: 

\begin{claim}\label{lemma_int}
If $\F, \Pi_\UP \models Q^{++}$, then $\F, \R \models Q$.
\end{claim}
\noindent\emph{Proof of claim.} Let $\UP_0, \UP_1, \ldots, \UP_k = \UP$ be
the sequence of sets of pattern definitions that led to $\UP$ in Step 3, with
$\UP_{i+1}$ being obtained from $\UP_i$ by a single direct (internal)
rewriting step. We prove by induction two distinct properties expressed at
rank $0 \leq j \leq k$:
\begin{description}
     \item[P1] every rule in $\Pi_{\UP_j}$ is a semantic
         consequence of $\Pi_{\UP_0} \cup \R$.
     \item[P2] for every fact base $\F'$ and CQ $Q'$ (over the original vocabulary):
     $$\F', \Pi_{\UP_j} \models (Q')^{++}\quad\Rightarrow\quad \F',
         \R \models Q'$$
\end{description}
In the second property, $(Q')^{++}$ denotes the CQ obtained by replacing every atom $p(t_1,t_2)$
such that $p$ is transitive by the atom $p^+(t_1,t_2)$. Observe that \textbf{P2} at rank $k$ yields the claim:
we simply take $\F'=\F$ and $Q' = Q$.

\smallskip

\emph{Base case ($i=0$):} property \textbf{P1} is obviously verified.
For property \textbf{P2}, we note that $\UP_0$ consists of the
following rules for every transitive predicate $p$: the transitivity rule
$p^+(x,y) \wedge p^+(y,z) \rightarrow p^+(x,z)$ and the initialization rule
$p(x,y) \rightarrow p^+(x,y)$. Clearly, if $\F, \Pi_{\UP_0} \models
(Q')^{++}$, then we have $\F, \R \models Q'$, since if we can derive $p^+(a,b)$
using $\F, \Pi_{\UP_0}$, then we can also derive $p(a,b)$ from $\F, \R$ using
the transitivity rule for $p$ in $\R$.

\smallskip

\emph{Induction step for \textbf{P1}:} we assume property \textbf{P1} holds for some rank $0 \leq i < k$
and show that it holds also for $i+1$.

Suppose that $\UP_{i+1}$ is obtained from $\UP_{i}$ by a
single direct rewriting step w.r.t.\ pattern name $P$ and the rule $R = B
\rightarrow H \in \R_L^+$. Let $\Q=P^+[x,y]$, $Q$ be the considered
instance of interest of $\Q$ w.r.t.\ $R$, $\mu = (Q',H,P_u)$ be the considered internal
unifier of $Q$ with $H$, and $\sigma$ be the considered substitution
associated with $\mu$ that preserves the external terms. Finally, let $B'$ be
obtained from $\sigma(B)$ by substituting the first (resp.\ second) external
term by $\#1$ (resp.\ $\#2$).

Since we know that $\mu$ is an internal unifier,
the external terms of $Q'$ cannot be unified together or with an existential variable. 
Thus by considering $Q''$ and $P_{u'}$ obtained from $Q'$ and $P_u$ by substituting
the first (resp.\ second) external term by $\#1$ (resp.\ $\#2$),
it is clear that $\mu' = (Q'', H, P_{u'})$ is a unifier of
$Q''$ with $R$ such that $\sigma'(B) = B'$, where
$\sigma'$ is the substitution associated with $\mu'$ that preserves
the special terms $\#1$ and $\#2$.

We consider two cases depending on the nature
of $B'$.

\smallskip

\noindent{\em Case 1:} The first possibility is that $B'$ is an atom
(as opposed to a repeatable pattern),
in which case we add the following rule to $\Pi_{\UP_i}$:
$B' \rightarrow p^+(\#1,\#2)$.

Let $a_1,\dots,a_k$ be the atoms of $Q''$,
and let $a'_j$ be the atom in $P$'s definition from which
$a_j$ is obtained.
Since there is a rewriting of $\{a_j~|\ 0 < j \leq k\}$ with $R$ 
into $B'$ (using the unifier $\mu')$, and the rule $R$ appears in the original set of rules $\R$,
it follows that $$\R \models B' \rightarrow a_1 \wedge \ldots \wedge a_k$$
From the induction hypothesis, we know that the rules
$a'_j \rightarrow p^+(\#1,\#2)$ ($0 < j \leq k$) are entailed
by $\Pi_{\UP_0},\R$.
We also know that for all $1 < j \leq k$, the atoms $a_{j-1}$ and $a_j$ share
a variable corresponding respectively to $\#2$ in $a'_{j-1}$
and to $\#1$ in $a'_j$. Thus, by applying the rules $a'_j \rightarrow p^+(\#1,\#2)$ ($0 < j \leq k$)
to the conjunction $a_1 \wedge \ldots \wedge a_k$, we obtain
$p^+(\#1,x_1) \wedge p^+(x_1,x_2) \wedge \dots \wedge p^+(x_{k-1},\#2)$.
Hence:
{\small$$ \Pi_{\UP_0},\R \models \bigwedge\limits_{j = 0}^{k}a_j  \rightarrow p^+(\#1,x_1) \wedge p^+(x_1,x_2) \wedge \dots \wedge p^+(x_{k-1},\#2)$$}
\hspace*{-.9mm}Since $\Pi_{\UP_0}$ contains a transitivity rule for $p^+$,
we can further infer that
$$\Pi_{\UP_0} \models  p^+(x_1,x_2) \wedge \dots \wedge p^+(x_{k-1},\#2) \rightarrow p^+(\#1,\#2)$$
By chaining together the preceding entailments, we obtain
$\Pi_{\UP_{0}},\R \models (B' \rightarrow p^+(\#1,\#2))$, as desired.

\smallskip

\noindent{\em Case 2:} The other possibility is that
$B'$ is a repeatable pattern of the form
$S^+[\#1, \#2]$ or $S^+[\#2, \#1]$.
Let $f$ be a bijection on $\{\#1,\#2\}$:
if $B'$ is of the form $S^+[\#1, \#2]$, $f$ is the identity, otherwise  $f$ permutes $\#1$ and $\#2$.
Then for all $s_\ell$ in the definition of $S$, we add $f(s_\ell)$ to $P$'s definition,
and we add the corresponding rules $f(s_\ell) \rightarrow p^+(\#1,\#2)$ to $\Pi_{\UP_i}$. 
Consider one such rule rule $f(s_\ell) \rightarrow p^+(\#1,\#2)$.

Let $a_1,\dots,a_k$ and $a'_1, \ldots, a'_k$ be defined as in Case 1.
Since there is a rewriting of $\{a_j~|\ 0 < j \leq k\}$ with $R \in \R_L^+$ 
into $B'$, and since the rule $R$ was obtained from a rule $R'$ in $\R$
by replacing the transitive predicate $s$ in the rule head by the repeatable pattern $S^+$,
it follows that $$\R \models f(s(\#1, \#2)) \rightarrow a_1 \wedge \ldots \wedge a_k$$
Arguing as in Case 1, we obtain
$$\Pi_{\UP_{0}},\R \models f(s(\#1, \#2)) \rightarrow p^+(\#1,\#2)$$
From the induction hypothesis, we know that that the rules
$s_\ell \rightarrow s^+(\#1,\#2)$ are entailed from
$\Pi_{\UP_0},\R$, and the same obviously holds for the rules $f(s_\ell) \rightarrow f(s^+(\#1,\#2))$.
By combining the preceding entailments, we obtain
$\Pi_{\UP_{0}},\R \models f(s_\ell) \rightarrow p^+(\#1,\#2)$.

\medskip

\emph{Induction step for property \textbf{P2}:} we assume \textbf{P2} holds for some rank $0 \leq i < k$
and show that it holds also for $i+1$.

Suppose now that $\F', \Pi_{\UP_{i+1}} \models (Q')^{++}$, for some fact base $\F'$
and CQ $Q'$ (over the original predicates).
This means that there is
a finite derivation sequence  $\F' = \F^{++}_0, \ldots, \F^{++}_m$ such that $\F^{++}_m
\models (Q')^{++}$ and such that for all $ 0 \leq \ell < m$, $\F^{++}_{\ell+1}$ is
obtained from $\F^{++}_\ell$ either (i) by a sequence of applications of rules from
$\Pi_{\UP_i}$ or (ii) by a sequence of applications of rules from $\Pi_{\UP_{i+1}}
\setminus \Pi_{\UP_i}$.

In case (i), we have 
$\F_\ell^{++}, \Pi_{\UP_i} \models \F_{\ell+1^{++}}$. Letting $\F_r$ be the
fact base obtained by replacing every predicate $p^+$ in $\F_r^{++}$
by the corresponding predicate $p$, and recalling that $\Pi_{\UP_i} $ contains the
rule $p(x,y) \rightarrow p^+(x,y)$, we have $\F_\ell, \Pi_{\UP_i} \models \F_{\ell+1}^{++}$.
Applying the induction hypothesis (treating $\F_{\ell+1}^{++}$ as a CQ),
we obtain $\F_\ell, \mathcal R \models \F_{\ell+1}$.

In case (ii), we have $\F_\ell^{++}, (\Pi_{\UP_{i+1}}
\setminus \Pi_{\UP_i}) \models \F_{\ell+1}^{++}$.
From property \textbf{P1}, we obtain $\F_\ell^{++}, \Pi_{\UP_{0}},
\mathcal R \models \F^{++}_{\ell+1}$. Using the rule
$p(x,y) \rightarrow p^+(x,y)$ (that is present in $\Pi_{\UP_{0}}$),
the latter yields $\F_\ell, \Pi_{\UP_{0}},
\mathcal R \models \F^{++}_{\ell+1}$. Finally,
we note that if we can
derive $p^+(a,b)$ from $\F_\ell, \Pi_{\UP_{0}},
\mathcal R$, then we can also infer $p(a,b)$
from $\F_\ell,
\mathcal R$ by using the transitivity rule for $p$
instead of using $p(x,y) \rightarrow p^+(x,y)$
and the transitivity rule for $p^+$.
Thus, we have $\F_\ell,
\mathcal R \models \F_{\ell+1}$.

We have thus shown that for every $ 0 \leq \ell < m$,
$\F_\ell, \mathcal R \models \F_{\ell+1}$.  Since $\F' = \F_0$,
by chaining these implications together, we obtain
$\F', \R \models \F_m$. Using the same reasoning as above,
we can infer $\F_m \models Q'$ from $\F^{++}_m \models (Q')^{++}$.
Then, by combining these statements, we obtain
 $\F', \mathcal R \models Q'$.
\emph{(end proof of claim)}

\medskip

Now let $\UQ$ be the set of queries computed in Step 5 by performing  
all possible external direct rewritings w.r.t.\ $\UP$ and rules from $\R_L^+$ , starting from $Q^+$,
and let $Q_\UQ$ be the set of CQs associated with $\UQ$ (defined as in Step 6).
We start by proving the following claim, which relates external direct rewriting steps
to sequences of classical direct rewritings.

\begin{claim}\label{preservation}
Let $\Q_{i+1}$ be a direct rewriting of $\Q_i$ w.r.t.\ $\UP$.
Then every full instance of $(\Q_{i+1},\UP)$ is
obtained from a sequence of (classical) direct rewritings of some full
instance of $(\Q_i,\UP)$.
\end{claim}

\noindent\emph{Proof of claim.} 
Let $(\Q_{i+1},\UP)$ be
obtained from an external rewriting of $(\Q_i,\UP)$ with rule
$R = B \rightarrow H$.
This means that there is a minimally unifiable instance $\Q^e$ 
and a
unifier $\mu = (X, H, P_u)$ of $\Q^e$ with $H$ (with associated substitution
$\sigma)$ such that $\Q_{i+1} = \sigma(\Q^e \setminus X) \cup \sigma(B)$.

Let us
consider a partial instance $\Q^P_{i+1}$ of $(\Q_{i+1},\UP)$ that fully
instantiates $\sigma(\Q^e \setminus X)$ but does not instantiate $\sigma(B)$
(we say that it is a $\sigma(B)$-excluding instance). Note that $\Q^P_{i+1}$
can be built equivalently by choosing a full instance $Q^e$ of $(\Q^e, \UP)$,
removing the atoms of $X$, then by applying the substitution $\sigma$ and
adding $\sigma(B)$. We can see that the classical direct rewriting of $Q^e$
according to $\mu$  produces $\Q^P_{i+1}$.
Moreover, since every full instance of $(\Q^e,\UP)$ is a full
instance of $(\Q_{i},\UP)$,  
we know that $Q^e$ is an instance of $(\Q_i,\UP)$.

Now consider any full instance $Q_{i+1}$ of $(\Q_{i+1},\UP)$. Note that it is a full
instance of some $\sigma(B)$-excluding instance $(\Q^P_{i+1},
\UP)$. There are two cases to consider:
\begin{itemize}
\item If $\sigma(B)$ is an atom, then $Q_{i+1} = \Q^P_{i+1}$ and thus $Q_{i+1}$ is
obtained from a classical direct rewriting of an instance of $(\Q_i,\UP)$.
\item Otherwise, if $\sigma(B)$ is a repeatable pattern, then $Q_{i+1}$ is obtained from
$(\Q^P_{i+1}, \UP)$ by expanding $\sigma(B)$ into a sequence
of $k$ standard patterns, and expanding each of them into some atom $a_\ell$.
Let $B_k = \{ a_\ell~|\ 1 \leq \ell \leq k \}$.
Then, $\sigma(B)$ is generated in forward chaining from $B_k$ with a
sequence of applications of rules: $k$ applications of transitivity rules,
and $k$ applications of the rules encoded in $\UP$, each one stemming from a
finite sequence of applications of rules of $\R$ (see Claim \ref{lemma_int}).
Thus from the completeness of classical rewriting, $B_k$ can
be obtained from a sequence of classical direct rewritings from $\sigma(B)$,
and thus $Q_{i+1}$ is obtained from a sequence of classical direct rewritings
of an instance of $(\Q_i,\UP)$.
\end{itemize}
\emph{(end proof of claim)}

\medskip

The following claim shows the soundness of the external rewriting in Step 5 and completes the proof of the lemma.

\begin{claim}\label{lemma_ext}
If $\F, \Pi_\UP \models Q_\UQ$, then $\F, \R \models Q$.
\end{claim}
\noindent\emph{Proof of claim.}
Suppose that $\F, \Pi_\UP \models Q_\Q$ with $\Q \in \UQ$.
We know that 
the PCQ $\Q$ is obtained from a finite sequence $\Q_0 = Q^+, \Q_1, \ldots, \Q_k = \Q$ of
PCQs such that for all $0 \leq j < k$, $(\Q_{j+1},\UP)$ is a direct 
external rewriting of $(\Q_j,\UP)$. We will show by induction on $j$
that $\F, \Pi_\UP \models Q_{\Q_j}$ implies $\F, \R \models Q$ for every $0 \leq j \leq k$.

The base case ($j=0$) is a direct consequence of
Claim~\ref{lemma_int}. For the induction step, we assume the property is true at rank~$i$,
and we show that it is true at rank $i+1$.

Suppose that $\F, \Pi_\UP \models Q_{\Q_{i+1}}$. From
Proposition \ref{prop:Datalog}, it follows that there is a full instance $Q_{i+1}$
of $(\Q_{i+1},\UP)$ such that $\F \models Q_{i+1}$.
By Claim~\ref{preservation}, there is a full instance $Q_i$ of
$(\Q_i,\UP)$ such that $Q_{i+1}$ is obtained from a sequence of classical
rewritings from $Q_i$. Thus (from the correctness of the classical
rewriting), there is a fact base $\F'$ such that $\F, \mathcal R \models \F'$
and $\F' \models Q_i$. Applying Proposition \ref{prop:Datalog}, we obtain $\F', \Pi_\UP
\models Q_{\Q_i}$. Now from our induction hypothesis, it follows that $\F',
\R \models Q$, hence 
$\F, \R \models Q$.
\emph{(end proof of claim)}
\end{proof}

\medskip

\noindent\textbf{Theorem \ref{thm:correct}}
\emph{
	Let $Q$ be a CQ, $(\F,\R)$ be a {\ltrans} KB, 
	and ($\Pi_{\UP}$,$Q_{\UQ}$) be the output of the algorithm. 
	Then: $\F,\R \models Q$ iff $\F,\Pi_{\UP} \models Q'$ for some $Q' \in Q_{\UQ}$.
}\\[1mm]
\begin{proof}
	Follows from Lemma \ref{prop:complete} and Lemma \ref{prop:sound}.
\end{proof}

\medskip
The following two lemmas show that the non-excluded minimally-unifiable instances
are sufficient to ensure completeness when the input query is atomic
or when the input rule set satisfies the safety condition.

\begin{lemma}
\label{prop:unif_1with2}
	Let $(\Q,\UP)$ be a PCQ, $R \in R^+_L$, $Q$ be an instance of interest
	of $(\Q,\UP)$ and
	$\mu = (Q',H,P_u)$ be an external unifier of $Q$ with $R$ such that
	two external terms w.r.t.\ $\mu$ from a given pattern $P^+[t_1,t_2]$
	are unified together and with no existential variable.

	Every minimally-unifiable instance $(\Q,\UP)$ w.r.t.\ $\mu$ that
	replaces $P^+[t_1,t_2]$ as in the External Rewriting cases $(i)$, $(ii)$,
	or $(iii)$ will lead to a direct rewriting $(\Q'_i,\UP)$ that is more
	specific than $(\Q,\UP)$.
	Furthermore, for any classical direct rewriting $\Q''_i$ of $\Q'_i$ with $R$,
	either $(\Q,\UP) \geq (\Q''_i,\UP)$ or there is a classical direct rewriting $\Q'$ of
	the minimally-unifiable instance of $\Q$ that replaces $P^+[t_1,t_2]$
	as in case $(iv)$ with $R$ and $(\Q',\UP) \geq (\Q''_i,\UP)$.

\end{lemma}

\begin{proof}
	Without loss of generality, let us write $\Q = q[t_1,t_2] \wedge P^+[t_1,t_2]$
	where $q[t_1,t_2]$ denotes a set of atoms where $t_1$ and $t_2$ may occur.
	We denote by $x_s$ and $x_e$ ($s < e$) the external terms of $P^+[t_1,t_2]$
	w.r.t.\ $\mu$, and by $A[x_s,x_e]$ the sequence of atoms
	expanded from $P^+[t_1,t_2]$
	involved in the unifier.
	Since we assume that no existential variable is unified with variables $x_s$ and $x_e$,
	no atom from $q$ can be part of the unifier.
	Consider the following minimally-unifiable instances:
	\begin{enumerate}
		\item $\Q_1 = q[t_1,t_2] \wedge P^+[t_1,x_s] \wedge A[x_s,x_e] \wedge P^+[x_e,t_2]$
		\item $\Q_2 = q[t_1,t_2] \wedge [t_1 = x_s,x_e] \wedge P^+[x_e,t_2]$
		\item $\Q_3 = q[t_1,t_2] \wedge P^+[t_1,x_s] \wedge X[x_s,x_e = t_2]$
	\end{enumerate}
	By unifying $x_s$ and $x_e$ together, we obtain the following instances:
	\begin{enumerate}
		\item $q[t_1,t_2] \wedge P^+[t_1,x_s] \wedge A[x_s,x_s] \wedge P^+[x_s,t_2]$
		\item $q[t_1,t_2] \wedge A[t_1,t_1] \wedge P^+[t_1,t_2]$
		\item $q[t_1,t_2] \wedge P^+[t_1,t_2] \wedge A[t_2,t_2]$
	\end{enumerate}
	Let $\Q'_i$ be the direct rewriting of $\Q_i$ w.r.t.\ $\mu$ with $R$.
	It is easy to see that $\Q \subseteq \Q'_2$ and $\Q \subseteq \Q'_3$,
	thus, $(\Q'_2, \UP)$ and $(\Q'_3,\UP)$ are more specific than $(\Q,\UP)$.
		
	Let $Q_1$ be a full instance of $(\Q'_1,\UP)$. We construct a full instance $Q$ of $(\Q,\UP)$ as follows.
	First note that $q[t_1,t_2]$ is common to both $\Q'_1$ and $\Q$,
	so we will expand all patterns in $q[t_1,t_2]$ exactly as in $Q_1$.
	Now let $k_1$ (resp.\ $k_2$) be the number of children of
	$P^+[t_1,x_s]$ (resp.\ $P^+[x_s,t_2]$) in the instantiation of $Q_1$,
	and expand $P^+[t_1,t_2]$ in $Q$ into $k = k_1 + k_2$ children:
	$P[t_1 = x_0,x_1], \dots, P[x_{k-1},x_{k} = t_2]$.
	Expand each $P[x_i,x_{i+1}]$ with $i < k_1$ as
	is expanded the $i^{th}$ child of $P^+[t_1,x_s]$ in $Q_1$; and
	each $P[x_i,x_{i+1}]$ with $k_1 \leq i < k$ as is expanded the
	$(i - k_1+1)^{th}$ child of $P^+[x_s,t_2]$ in $Q_1$.
	By construction, there is an homomorphism from $Q$ to $Q_1$.
	We have thus shown that  $(\Q,\UP) \geq (\Q'_1,\UP)$.

	Furthermore, let $\Q''_i$ be a classical direct rewriting
	of $\Q'_i$ with a rule $R' = B' \rightarrow H'$ w.r.t.\ unifier $\mu' = (Q', H', P'_u)$, where $1 \leq i \leq 3$.
	If at least one atom involved in $\mu'$ occurs in $\Q'_i
	\setminus \sigma(B)$ (where $\sigma$ is the substitution associated with $\mu$),
	then, let $\mu'' = \{Q'', H', P''_u\}$ where $Q'' = Q'
	\setminus \sigma(B)$ and
	$P''_u$ is the restriction of $P'_u$ to terms occurring in $Q'' \cup H'$.
	Since $Q'' \neq \emptyset$ and
	all terms from $\sigma(B)$ cannot connect
	two different terms from $q[t_1,t_2]$
	(indeed, the only term shared between $\sigma(B)$ and $q[t_1,t_2]$ is either
	 $t_1$ or $t_2$), $\sigma(B)$ can be seen as a loop on $t_1$ (or $t_2$),
	therefore we can remove $\sigma(B)$ while preserving the unifier, i.e.,
	$\mu''$ is a unifier of $\Q$ with $R'$.
	Moreover, since $P''_u$ and $Q''$ are only restrictions of
	$P'_u$ and $Q'$ respectively,
	it holds that $\mu'' \geq \mu'$.
	Then, we denote by $\Q''$ the direct rewriting of $\Q$ with $R'$ w.r.t.\ $\mu''$ and obtain $\Q'' \geq \Q_i'$.
	The other possibility is that all atoms involved in
	$\mu'$ occur in $\sigma(B)$, then,
 	$\Q_2''$ (resp.\ $\Q_3''$) is more specific than $\Q$
	since $\Q \subseteq \Q_2''$ (resp.\ $\Q \subseteq \Q_3''$).
	Moreover, for any instance $Q''_1$ of $\Q_1''$, one can easily build an
	instance $Q'$ of $\Q$ in the same way as above, and see that
	$Q' \geq Q''_1$.  Thus, we have $(\Q,\UP) \geq (\Q''_i,\UP)$.
\end{proof}

\begin{lemma}
\label{prop:unif_1withE}
	Let $(\Q,\UP)$ be a PCQ, $R \in \R^+_L$,
	$Q$ be an instance of interest of $(\Q,\UP)$ and
	$\mu = (Q',H,P_u)$ be an external unifier of $Q$ with $R$ such that
	one external term w.r.t.\ $\mu$ from a given pattern $P^+[t_1,t_2]$
	is unified with an existential variable,
	and where all atoms in $Q'$ are obtained from the expansion of a repeatable pattern.

	If $\Q$ is atomic, or if $\R_L$ is a set of safe linear rules, then
	every minimally-unifiable instance of $(\Q,\UP)$ w.r.t.\ $\mu$ that
	replaces all $P_i^+[t_1,t_2]$ as in the External Rewriting cases $(ii)$
	or $(iii)$ will lead to a direct rewriting $(\Q_i',\UP)$ that is more
	specific than $(\Q,\UP)$.
	Furthermore, for any direct rewriting $\Q''_i$ of $\Q'_i$ with $R$,
	either $\Q \geq \Q''_i$ or there is a direct rewriting $\Q'$ of
	a minimally-unifiable instance of $\Q$ w.r.t.\ $\mu$ that replaces
	at least one repeatable pattern as in case $(iv)$ and is
	such that $\Q' \geq \Q''_i$.

\end{lemma}

\begin{proof}
	Let $(\Q,\UP)$, $R$, $Q$ and $\mu$ be as in the lemma statement, and let
	$P_1^+[t_1^1,t_2^1], \ldots, P_k^+[t_1^k,t_2^k]$ be the repeatable patterns that are relevant for $\mu$.
	For each $1 \leq i \leq k$,
	we denote by $P_i[t_1^i = x_0^i, x_1^i], \dots, P_i[x_{k-1}^i, x_{k_i}^i = t_2^i]$ the sequence
	of standard patterns expanded from $P_i^+[t_1^i,t_2^i]$, and we
	let $x^i_{s_i}$ and $x^i_{e_i}$ ($s_i < e_i$) be the external terms of $P_i^+[t_1^i,t_2^i]$ w.r.t.\ $\mu$.
	We assume without loss of generality that it is $x^i_{e_i}$ that is unified with an existential variable,
	and let $A_i[x^i_{s_i},x^i_{e_i} = t_2^i]$ denote the atoms expanded from $P_i[x_j^i,x_{j+1}^i]$ with
	$s_i \leq j < e_i$.

	Since the unifier $\mu$ is single-piece, all repeatable patterns relevant for $\mu$ have
	to share some variable. For simplicity, we assume that they all share their second term, i.e.
	 $t_2^i = t_2^j$ for all $1 \leq i,j \leq k$. (The argument is entirely similar, just more notationally involved,
	  if this assumption is not made.)
	Let us use $t_2$ for this shared term.
	Then we can write $\Q$ as follows:
	$$\Q = q[t_1^1,\dots,t_1^k] \wedge  \bigwedge\limits_{1 \leq i \leq k} P^+_i[t_1^i,t_2]$$
	Note that $t_2$ cannot occur in $q$.
	
	Because we have chosen the second term to be shared in all repeatable patterns,
	we only need to consider the minimally-unifiable instance $\Q_M$ of $(\Q,\UP)$ w.r.t.\ $\mu$ that replaces
	each $P^+_i[t_1^i,t_2]$ by $P_i^+[t_1^i,x^i_s], A_i[x^i_s,x^i_e = t_2]$, i.e.\ External Rewriting case (ii).
	Thus, we have
	 $$\Q_M = q[t_1^1,\dots,t_1^k] \wedge \bigwedge\limits_{1 \leq i \leq k} (P_i^+[t_1^i,x^i_s] \wedge A_i[x^i_s,t_2]).$$
	Let $\sigma$ be the substitution associated
	with $\mu$.
	From the safety condition (see Section \ref{sec:safety-condition}), we know that there is
	a pair of positions $\{p_1,p_2\}$ for the predicate $p$ of $H$, such that for all atoms $p(\t)$ occurring in
	a pattern definition the terms $\#1$ and $\#2$ occurs in positions $\{p_1,p_2\}$.
	We further note that the external terms in the concerned patterns are $t_2$ (which unifies with an existential
	variable in $H$) and the terms $x^i_s$ (which unify with a non-existential variable),
	and each of these external terms must be obtained by instantiating term $\#1$ or $\#2$.
	Since the $A_i[x^i_s,t_2]$ are unified together, and share the same predicate $p$,
	it follows that all of the $x^i_s$ must occur in the same position (either $p_1$ or $p_2$) of $p$;
	$t_2$ will occur in the other position among $p_1$ and $p_2$. We therefore obtain;
	$$\sigma(x^1_s) = \sigma(x^2_s) = \dots = \sigma(x^k_s) = x',$$
	where $x'$ is the term in $B$ that unifies with all of the $x^i_s$.
	(Note that if $\Q$ is an atomic query, there is a single $A_i$, so the previous statement
	 obviously holds, even without the safety condition.)
	Thus, $\Q_M$ becomes:
	$$q[t_1^1,\dots,t_1^k] \wedge \bigwedge\limits_{1 \leq i \leq k} (P_i^+[t_1^i,x'] \wedge A_i[x',t_2]).$$
	There is an isomorphism from $\Q$ to $\Q_M \setminus \{A_i \mid 1 \leq i \leq k\}$ that maps $t_2$ to
	$x'$.
	We then observe that $\{A_i \mid 1 \leq i \leq k\}$ is exactly the set of atoms that will be erased in
	the direct rewriting $\Q'_M = \Q_M \setminus \{A_i \mid 1 \leq i \leq k \} \cup \sigma(B)$,
	where $\sigma$ is a substitution associated with $\mu$.
	Therefore, $\Q$ is isomorphic to $\Q'_M \setminus \sigma(B)$, hence $(\Q,\UP) \geq (\Q'_M,\UP)$.
	One can see that the same reasoning as in the previous proof can be applied here
	to show that any further direct rewriting $\Q''_M$ of $\Q'_M$ will lead to more
	specific queries.
\end{proof}

\medskip

\noindent\textbf{Theorem \ref{prop:completeness-excluded}}
\emph{
	The modified query rewriting algorithm  halts.
	Moreover, Theorem \ref{thm:correct} (soudness and completeness)
	holds for the modified 
	algorithm if either
	the input CQ is atomic, or the input rule set is safe.
}

\begin{proof}
	From Lemma \ref{prop:complete}, we know that
	if we do not exclude any rewriting the algorithm is sound
	and complete,
	and Lemma \ref{prop:unif_1with2} and \ref{prop:unif_1withE} show that
	for any rewriting $\Q$ that we exclude, there is another
	rewriting $\Q'$ obtainable using only
	non-excluded direct rewritings that is more general
	than $\Q$.
	Therefore, the modified algorithm (in case of an atomic CQ,
	or a safe rule set) is complete.
	Furthermore, excluding rewritings cannot comprise the soundness
	of the rewriting mechanism.
\end{proof}

\medskip

\noindent\textbf{Theorem \ref{thm:data}}
\emph{
	Both (i) atomic CQ entailment over \ltrans\ KBs and (ii) CQ entailment over safe \ltrans\ KBs
	are  \NL-complete in data complexity. }\\[1mm]
\begin{proof}
	Consider a CQ $Q$,  a \ltrans\ rule set $\R$,
	and a fact base $\F$. Suppose that either $Q$ is atomic or $\R$ satisfies the safety condition. 
	Using Theorem \ref{prop:completeness-excluded}, we can compute
	a finite set $\Pi_\UP$  of Datalog rules and a finite set $Q_\UQ$ of CQs with the property that
	 $\F,\R \models Q$ iff $\F,\Pi_{\UP} \models Q'$ for some $Q' \in Q_{\UQ}$.
	 As $\Pi_\UP$ and $Q_\UQ$ do not depend on the fact base $\F$, they can be computed and stored
	 using constant space w.r.t.\ $|\F|$.

	To test whether $\F,\Pi_{\UP} \models Q'$ for some $Q' \in Q_{\UQ}$, we proceed as follows.
	For each rewriting $Q' \in Q_{\UQ}$,
	we can consider every possible mapping $\pi$ from the variables of $Q'$ to the terms of $\F$.
	We then check whether the facts in $\pi(Q')$ are entailed from $\F,\Pi_{\UP}$.
	For every atom $\alpha \in Q'$ over one of the original predicates, we can directly check if $\pi(\alpha) \in \F$,
	since the rules in $\Pi_\UP$ can only be used to derive facts over the new predicates $p^+$.
	For every atom $p^+[t_1,t_2] \in Q'$ where $p^+$ is a new predicate,
	we need to check whether $\F,\Pi_{\UP} \models p^+(\pi(t_1),\pi(t_2))$.
	Because of the shape of the rules in $\Pi_{\UP}$, the latter holds just in the case that there is
	a path of constants $c_1, \ldots, c_n$ with $c_1=\pi(t_1)$ and $c_n=\pi(t_2)$ such that
	for every $1 \leq i < n$, there is a rule $\rho_i = B_i \rightarrow p^+(\# 1, \#2)$ and substitution $\sigma_i$
	of the variables in $B_i$ by constants in $\F$ such that $\sigma_i(\#1)=c_i$, $\sigma_i(\#2)=c_{i+1}$, and $\sigma_i(B_i) \in \F$.	
	To check for the existence of such a path, we guess the constants $c_i$ in the path one at a time, together with the witnessing rule $\rho_i$
	and substitution $\sigma_i$, 
	using a counter to ensure that the number of guessed constants
	does not exceed the number of constants in $\F$. Note that we need only logarithmically many bits
	for the counter, so the entire procedure runs in non-deterministic logarithmic space.
	
	Hardness for \NL\ can be shown by an easy reduction from the \NL-complete directed reachability problem.
\end{proof}

\medskip

\noindent\textbf{Theorem \ref{thm:combined}}
\emph{Both (i) atomic CQ entailment over \ltrans\ KBs and (ii) CQ entailment over safe \ltrans\ KBs
	are in ExpTime in combined complexity.
	Furthermore, atomic CQ entailment over \ltrans\ KBs is
	ExpTime-hard in combined complexity.}\smallskip 

The proof of Theorem \ref{thm:combined} is provided in the following two lemmas.

\begin{lemma}
	Both (i) atomic CQ entailment over \ltrans\ KBs and (ii) CQ entailment over safe \ltrans\ KBs
	are in ExpTime in combined complexity.
\end{lemma}

\newcommand{\PP}{\UP}
\newcommand{\PQ}{\UQ}

\begin{proof}
	Consider a CQ $Q$, a {\em linear}+\fun{trans}
	rule set $\R = \R_L \cup \R_T$,
	with $\R_L$ a set of linear rules and $\R_T$ a set of transitivity rules,
	and a set of facts $\F$. 
	Suppose that either condition (i) or (ii) of the lemma statement holds.
	It follows from Theorem \ref{prop:completeness-excluded} that the modified query rewriting algorithm halts
	and returns a finite set $\Pi_\PP$  of Datalog rules
	and a finite set $Q_\PQ$ of CQs
	such that $(\F,\R) \models Q$ iff $(\F,\Pi_{\PP}) \models Q'$ for some $Q' \in Q_{\PQ}$.

	To prove membership in ExpTime, we show that:
	\begin{enumerate}
		\item[(i)] $\Pi_{\PP}$ is of exponential size and can be built in exponential time;
		\item[(ii)] $Q_{\PQ}$ is a set of exponential size, that can
		be built in exponential time, and any
		$Q' \in Q_{\PQ}$ is of linear size in $Q$;
		\item[(iii)] we can saturate $\F$ with $\Pi_{\PP}$ into $\F^*$ in polynomial time in the size of $\Pi_{\PP}$ and $\F$,
		and the resulting set of facts is of
		polynomial size in $\F$;
		\item[(iv)] $Q_{\PQ}$ can be evaluated over $\F^*$ in
		exponential time.
	\end{enumerate}

	We denote by $r$ the maximum arity of a predicate in $\R$,
	by $p$ the number of predicates occurring in $\R$ and by
	$t$ the number of transitive predicates.
	
	Let us consider the construction of $\Pi_{\PP}$.
	Since all rules generated in this step are linear rules
	and given a predicate $s$ the number of non-isomorphic atoms
	using $s$ is bounded by an exponential
	in $r$,
	for each transitive predicate there can be only exponentially many
	generated rules.
	Thus $|\Pi_{\PP}| = O(t\times p \times r^{r})$.
	For the first point, it remains to show that $\Pi_{\PP}$ can be
	built in exponential time.
	Consider the following algorithm:
	for each pattern definition $P$, repeat until fixpoint:
	choose a rule $R = (B,H) \in \R_L$,
	compute all instances of interest
	of $P$ w.r.t.\ $R$, and if there is an internal unifier, add the
	corresponding rewriting to $P$'s definition.
	The repeatable pattern $P^+[t_1,t_2]$ can be expanded
	into at most $r+2$ standard patterns
	(by the definition of instances of interest), and thus there
	are $r+2$ possible sizes for the instances of interest.
	Then for each of these standard patterns, we can choose an atom
	from $P's$ definition that uses the predicate of $H$.
	Since there are at most $r^r$ possible choices
	for instantiating a standard pattern, and there are at most
	$r+2$ standard patterns to expand, we obtain the following bound:
	there are $O((r+2) \times (r^{r})^{r+2}) = O(r^{r^2})$
	different instances of interest for a given pattern definition
	and a given rule.
	Therefore each step of the algorithm can be processed in
	exponential time.
	Since there are only exponentially many different possible 
	rewritings, the fixpoint is reached in at most exponential time.
	Hence, Point $(i)$ runs in exponential time.

	The argument for Point $(ii)$ proceeds similarly.
	The only difference comes from the fact that since $Q$ might not be
	atomic, we apply the rewriting step to conjunctive queries.
	However, from Proposition \ref{excluded},
	we know that all rewritten
	queries have size bounded by the size of $Q$.
	Therefore, by using the same argument as for Point $(i)$,
	we know that this step is exponential in both the maximum
	arity and in the size of the initial query $Q$. 

	Regarding Point $(iii)$, a single
	breadth-first step with all non-transitive 
	rules in $\Pi_{\PP}$ followed
	by the computation of the transitive closure is enough to build
	$\F^*$.
	While there are exponentially many non-transitive rules,
	each can be applied in polynomial time (since the body of each
	rule is atomic).
	Since each rule only creates atoms with transitive
	predicates, the resulting set of facts is of size $|terms(\F)|^2 \times
	p$.
	Now the transitive closure adds at most a quadratic number of
	atoms (for each transitive predicate), and can be computed in
	polynomial time in the size of $\F$.
	Therefore, $\Pi_{\PP}$ can be built in exponential time in $r$
	and is of polynomial size in $|\F|$.

	It remains to show that point $(iv)$ can be done in exponential
	time.
	Observe that since each query $Q' \in Q_{\PQ}$ is of size bounded
	by the initial query $Q$
	(Proposition \ref{excluded}),
	its evaluation can be computed in $NP$, thus in exponential
	time.
	Since there are only exponentially many queries in $Q_{\PQ}$,
	this step is also done in exponential time.

	Therefore, we can conclude that the entailment problem over
	{\em linear}+\fun{trans} sets of rules with atomic query, and
	over safe {\em linear}+\fun{trans} sets of rules is in ExpTime.
\end{proof}

\begin{lemma}
	Atomic CQ entailment over \ltrans\ KBs is
	ExpTime-hard in combined complexity.
\end{lemma}

\begin{proof}
	To prove hardness, we can rely on a proof from
	\cite{bt:rr16}.
	In this paper, they prove that
	Regular Path Query (RPQ) entailment 
	over linear knowledge
	bases is ExpTime-hard.
	The problem is not a subproblem of ours, nor the contrary.
	However the proof uses only a particular RPQ of the form
	$p^+(t_1,t_2)$.
	This RPQ is entailed from $(\F,\R_L)$ if and only if the atomic CQ
	$p(t_1,t_2)$ is entailed from $(\F, \R_L \cup \{trans(p)\})$.
	Nevertheless, we recall below the main lines of the proof, while
	reformulating it in terms of our problem.
	Note that the linear rules have a non-atomic head to simplify
	the explanations, but can be decomposed into atomic-headed
	without loss of generality.

	The reduction is from the simulation of any Alternating Turing
	Machine (ATM) that runs in polynomial space.
	More specifically, the problem they consider is the following
	ExpTime-complete problem:
	given a PSpace ATM $M$, and a word $x$, does $M$ accept $x$?
	Without loss of generality, they consider ATM where each
	non-final universal state has exactly two existential state
	successors, and
	each non-final existential state has exactly two universal
	state successors.

	The proof uses a
	single transitive predicate that we call $p$.
	Given an ATM $M$ with input $x$, we create a predicate of arity
	polynomial in $x$ and $M$, that encodes the current
	configuration of the machine (its tape
	and the current state and head position).
	Furthermore, each atom encoding a configuration also 
	uses a term as a
	``begin'' and another as an ``end''
	(respectively the first and last position of the predicate),
	these are used later by the
	transitivity rules.
	Linear rules are used to generate the transitions of the ATM.
	First, for each transition in the ATM, there is a linear rule
	that generates the two next configurations, and depending on
	the type of the current state different transitive atoms are
	generated as illustrated by Figure \ref{fig:lin-exphard}.

	The initial configuration contains two special constants
	$b$ and $e$ as begin and end,
	and the set of facts contains only the atom
	encoding this configuration.

	When the state of the current configuration $s$ is existential,
	four atoms using predicate $p$
	are generated in the next step, the first two being used to link
	the begin of $s$ to the 
	begin of the two next configurations (since the ATM is non-deterministic by nature),
	and the last two atoms being used to
	link the end of the two next configurations to the end of $s$.

	When the state of the current configuration $s$ is universal,
	three atoms using $p$ are
	generated, the first one links the begin of $s$ to the begin of
	the first next configuration, the second one links the end of the first
	next configuration to the begin of the second next configuration, and finally
	the last one links the end of the last next configuration to the end of
	$s$.

	Finally, when the state of the current configuration $s$ is accepting,
	an atom using $p$ linking the begin of $s$ with the end of $s$
	is generated.

	The idea is that linear rules simulate the run of the
	machine, and that transitivity rules connect the initial begin
	to the initial end if and only if $M$ accepts $x$.

	\newcommand{\conf}[2]{
			\draw[-,black] (#1,#2) -- (#1+3,#2) -- (#1+3,#2-0.5)
				-- (#1,#2-0.5) -- (#1,#2);
			\draw[-,black] (#1+0.5,#2) -- (#1+0.5,#2-0.5);
			\draw[-,black] (#1+1,#2) -- (#1+1,#2-0.5);
			\node[] at (#1+1.5,#2-0.25) {$\dots$};
			\draw[-,black] (#1+2,#2) -- (#1+2,#2-0.5);
			\draw[-,black] (#1+2.5,#2) -- (#1+2.5,#2-0.5);
	}

	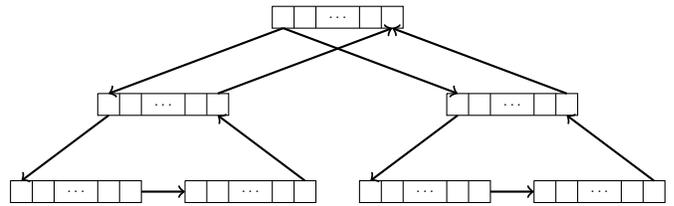
\begin{figure}
		\begin{center}
		\begin{tikzpicture}[every node/.style={transform shape},
		scale=0.58]
			\conf{0}{0}
			\conf{-4}{-2}
			\conf{4}{-2}
			\conf{-6}{-4}
			\conf{-2}{-4}
			\conf{2}{-4}
			\conf{6}{-4}

			\draw[thick,->,black] (0.25,-0.5) -- (-3.75,-2);
			\draw[thick,->,black] (0.25,-0.5) -- (4.25,-2);
			\draw[thick,->,black] (-1.25,-2) -- (2.75,-0.5);
			\draw[thick,->,black] (6.75,-2) -- (2.75,-0.5);

			\draw[thick,->,black] (-3.75,-2.5) -- (-5.75,-4);
			\draw[thick,->,black] (-3,-4.25) -- (-2,-4.25);
			\draw[thick,->,black] (0.75,-4) -- (-1.25,-2.5);

			\draw[thick,->,black] (4.25,-2.5) -- (2.25,-4);
			\draw[thick,->,black] (5,-4.25) -- (6,-4.25);
			\draw[thick,->,black] (8.75,-4) -- (6.75,-2.5);
		\end{tikzpicture}
		\end{center}

		\caption{Reduction from ATM simulation to 
			atomic CQ entailment over {\ltrans}
			knowledge bases. Edges stand for $p$-atoms and arrays
			stand for configuration atoms, with the first and last
			elements corresponding to the begin and end terms.}
		\label{fig:lin-exphard}
	\end{figure}

	Then, the query just asks whether the begin of the initial
	configuration can be linked to the end of the initial
	configuration (i.e., $Q = p(b,e)$).

	This reduction shows that atomic CQ entailment 
	over {\ltrans} sets of rules is ExpTime-hard.
\end{proof}

\end{document}